\documentclass{article}

\usepackage{arxiv}

\usepackage[utf8]{inputenc} 
\usepackage[T1]{fontenc}    
\usepackage{hyperref}       
\usepackage{url}            
\usepackage{booktabs}       
\usepackage{amsfonts}       
\usepackage{nicefrac}       
\usepackage{microtype}      

\usepackage{amsfonts,amssymb}
\usepackage{amsmath,amsxtra,bm}
\usepackage{xcolor}
\usepackage{amsbsy}
\usepackage{graphicx}
\usepackage{subfigure}
\usepackage{algorithm}
\usepackage{algorithmic}
\usepackage{multirow}
\graphicspath{{fig/}}

\newtheorem{proposition}{Proposition}[section]
\newtheorem{definition}{Definition}[section]
\newtheorem{remark}{Remark}[section]
\newenvironment{proof}{\textit{Proof:}}{\hfill$\square$\\}

\newcommand{\kpoly}{\mathcal{P}}
\newcommand{\pd}{\partial}
\newcommand{\mr}{\mathbb{R}}

\newcommand{\mn}{\mathbb{N}}
\newcommand{\symnet}{SymNet}

\title{PDE-Net 2.0: Learning PDEs from Data with A Numeric-Symbolic Hybrid Deep Network}

\author{
  Zichao Long \\ 
  School of Mathematical Sciences\\
  Peking University\\
  \texttt{zlong@pku.edu.cn} \\
   \And
 Yiping Lu  \\
  School of Mathematical Sciences\\
  Peking University\\
  \texttt{luyiping9712@pku.edu.cn} \\
  \AND
 Bin Dong \\
 Beijing International Center for Mathematical Research\\
 and Center for Data Science\\
 Peking University \\
 \texttt{dongbin@math.pku.edu.cn}
}
\begin{document}

\maketitle

\begin{abstract}
Partial differential equations (PDEs) are commonly derived based on empirical observations. However, recent advances of technology enable us to collect and store massive amount of data, which offers new opportunities for data-driven discovery of PDEs. In this paper, we propose a new deep neural network, called PDE-Net 2.0, to discover (time-dependent) PDEs from observed dynamic data with minor prior knowledge on the underlying mechanism that drives the dynamics. The design of PDE-Net 2.0 is based on our earlier work \cite{Long2018PDE} where the original version of PDE-Net was proposed. PDE-Net 2.0 is a combination of numerical approximation of differential operators by convolutions and a symbolic multi-layer neural network for model recovery. Comparing with existing approaches, PDE-Net 2.0 has the most flexibility and expressive power by learning both differential operators and the nonlinear response function of the underlying PDE model. Numerical experiments show that the PDE-Net 2.0 has the potential to uncover the hidden PDE of the observed dynamics, and predict the dynamical behavior for a relatively long time, even in a noisy environment.
\end{abstract}


\section{Introduction}

Differential equations, especially partial differential equations(PDEs), play a prominent role in many disciplines to describe the governing physical laws underlying a given system of interest. Traditionally, PDEs are derived mathematically or physically based on some basic principles, e.g. from Schr\"{o}dinger's equations in quantum mechanics to molecular dynamic models, from Boltzmann equations to Navier-Stokes equations, etc. However, the mechanisms behind many complex systems in modern applications (such as many problems in multiphase flow, neuroscience, finance, biological science, etc.) are still generally unclear, and the governing equations of these systems are commonly obtained by empirical formulas \cite{brennen2005fundamentals,efendievmathematical}. With the recent rapid development of sensors, computational power, and data storage in the last decade, huge quantities of data can now be easily collected, stored and processed. Such vast quantity of data offers new opportunities for data-driven discovery of (potentially new) physical laws. Then, one may ask the following interesting and intriguing question: can we learn a PDE model to approximate the observed complex dynamic data?

\subsection{Existing Work and Motivations}

Earlier attempts on data-driven discovery of hidden physical laws include \cite{bongard2007automated,schmidt2009distilling}. Their main idea is to compare numerical differentiations of the experimental data with analytic derivatives of candidate functions, and apply the symbolic regression and the evolutionary algorithm to determining the nonlinear dynamic system. When the form of the nonlinear response function of a PDE is known, except for some scalar parameters, \cite{raissi2018hidden} presented a framework to learn these unknown parameters by introducing regularity between two consecutive time step using Gaussian process. Later in \cite{raissi2017physicsII}, a PDE constraint interpolation method was introduced to uncover the unknown parameters of the PDE model. An alternative approach is known as the sparse identification of nonlinear dynamics (SINDy) \cite{brunton2016discovering,schaeffer2017learning,rudy2017data,chang2018identification,schaeffer2018extracting,wu2019learning}. The key idea of SINDy is to first construct a dictionary of simple functions and partial derivatives that are likely to appear in the equations. Then, it takes the advantage of sparsity promoting techniques (e.g. $\ell_1$ regularization) to select candidates that most accurately represent the data. In \cite{de2017deep}, the authors studied the problem of sea surface temperature prediction (SSTP). They assumed that the underlying physical model was an advection-diffusion equation. They designed a special neural network according to the general solution of the equation. Comparing with traditional numerical methods, their approach showed improvements in accuracy and computation efficiency.

Recent work greatly advanced the progress of PDE identification from observed data. However, SINDy requires to build a sufficiently large dictionary which may lead to high memory load and computation cost, especially when the number of model variables is large. Furthermore, the existing methods based on SINDy treat spatial and temporal information of the data separately and does not take full advantage of the temporal dependence of the PDE model. Although the framework presented by \cite{raissi2018hidden,raissi2017physicsII} is able to learn hidden physical laws using less data than the approach based on SINDy, the explicit form of the PDEs is assumed to be known except for a few scalar learnable parameters. The approach of \cite{de2017deep} is specifically designed for advection-diffusion equations, and cannot be readily extended to other types of equations. Therefore, extracting governing equations from data in a less restrictive setting remains a great challenge.

The main objective of this paper is to design a \textit{transparent} deep neural network to uncover hidden PDE models from observed complex dynamic data with minor prior knowledge on the mechanisms of the dynamics, and to perform accurate predictions at the same time. The reason we emphasize on both model recovery and prediction is because: 1) the ability to conduct accurate long-term prediction is an important indicator of accuracy of the learned PDE model (the more accurate is the prediction, the more confident we have on the underlying recovered PDE model); 2) the trained neural network can be readily used in applications and does not need to be re-trained when initial conditions are altered. Our inspiration comes from the latest development of deep learning techniques in computer vision. An interesting fact is that some popular networks in computer vision, such as ResNet\cite{he2016deep,he2016identity}, have close relationship with ODEs/PDEs and can be naturally merged with traditional computational mathematics in various tasks
 \cite{chen2015learning,weinan2017proposal,haber2017stable,sonoda2017double,Lu2018Beyond,chang2017multi,chen2018neural,qin2018data,wiewel2018latent,kim2018deep}.
However, existing deep networks designed in deep learning mostly emphasis on expressive power and prediction accuracy. These networks are not transparent enough to be able to reveal the underlying PDE models, although they may perfectly fit the observed data and perform accurate predictions. Therefore, we need to carefully design the network by combining knowledge from deep learning and numerical PDEs.

\subsection{Our Approach}

The proposed deep neural network is an upgraded version of our original PDE-Net \cite{Long2018PDE}. The main difference is the use of a symbolic network to approximate the nonlinear response function, which significantly relaxes the requirement on the prior knowledge on the PDEs to be recovered. During training, we no longer need to assume the general type of the PDE (e.g. convection, diffusion, etc.) is known. Furthermore, due to the lack of prior knowledge on the general type of the unknown PDE models, more carefully designed constraints on the convolution filters as well as the parameters of the symbolic network are introduced. We refer to this upgraded network as PDE-Net 2.0.

Assume that the PDE to be recovered takes the following generic form
\begin{equation*}
  U_t=F(U,\nabla U, \nabla^2 U,\ldots),\quad x \in \Omega\subset \mathbb{R}^2,\quad t\in [0,T].
\end{equation*}
PDE-Net 2.0 is designed as a feed-forward network by discretizing the above PDE using forward Euler in time and finite difference in space. The forward Euler approximation of temporal derivative makes PDE-Net 2.0 ResNet-like \cite{he2016deep,weinan2017proposal,Lu2018Beyond}, and the finite difference is realized by convolutions with trainable kernels (or filters). The nonlinear response function $F$ is approximated by a symbolic neural network, which shall be referred to as $\symnet$. All the parameters of the $\symnet$ and the convolution kernels are jointly learned from data. To grant full transparency to the PDE-Net 2.0, proper constraints are enforced on the $\symnet$ and the filters. Full details on the architecture and constraints will be presented in Section 2.

\subsection{Relation with Model Reduction}
Data-driven discovery of hidden physical laws and model reduction have a lot in common. Both of them concern on representing observed data using relatively simple models. The main difference is that, model reduction emphasis more on numerical precision rather than acquiring the analytic form of the model.

It is common practice in model reduction to use a function approximator to express the unknown terms in the reduced models, such as approximating subgrid stress for large-eddy simulation\cite{duraisamy2015new,duraisamy2019turbulence,ma2018model} or approximating interatomic forces for coarse-grained molecular dynamic systems\cite{noid2008multiscale,zhang2018deep}. Our work may serve as an alternative approach to model reduction and help with analyzing the reduced models.

\subsection{Novelty}

The particular novelties of our approach are that we impose appropriate constraints on the learnable filters and use a properly designed symbolic neural network to approximate the response function $F$. Using learnable filters makes the PDE-Net 2.0 more flexible, and enables more powerful approximation of unknown dynamics and longer time prediction (see numerical experiments in Section \ref{Sec:Results:Burgers} and Section \ref{Sec:Results:Heat}). Furthermore, the constraints on the learnable filters and the use of a deep symbolic neural network enable us to uncover the analytic form of $F$ with minor prior knowledge on the dynamic, which is the main advantage of PDE-Net 2.0 over the original PDE-Net. In addition, the composite representation by the symbolic network is more efficient and flexible than SINDy. Therefore, the proposed PDE-Net 2.0 is distinct from the existing learning based methods to discover PDEs from data.

\section{PDE-Net 2.0: Architecture, Constraints and Training}\label{Sec:PDE-Net}

Given a series of measurements of some physical quantities $\{U(t,x,y):\ t=t_0,t_1,\cdots,\ (x,y)\in\Omega\subset\mathbb{R}^2\}\subset\mr^d$ with $d$ being the number of physical quantities of interest, we want to discover the governing PDEs from the observed data $\{U(t,x,y)\}$. We assume that the observed data are associated with a PDE that takes the following general form:
\begin{equation}\label{E:PDE:General}
  U_t(t,x,y)=F(U,U_x,U_y,U_{xx},U_{xy},U_{yy},\ldots),
\end{equation}
here $U(t,\cdot): \Omega\mapsto \mathbb{R}^d$, $F(U,U_x,U_y,U_{xx},U_{xy},U_{yy},\ldots)\in\mr^d$, $(x,y) \in \Omega\subset\mathbb{R}^2$, $t\in [0,T]$. Our objective is to design a feed-forward network, called PDE-Net 2.0, to approximate the unknown PDE \eqref{E:PDE:General} from its solution samples in the way that: 1) we are able to reveal the analytic form of the response function $F$ and the differential operators involved; 2)  we can conduct long-term prediction on the dynamical behavior of the equation for any given initial conditions. There are two main components of the PDE-Net 2.0 that are combined together in the same network: one is automatic determination on the differential operators involved in the PDE and their discrete approximations; the other is to approximate the nonlinear response function $F$.
In this section, we start with discussions on the overall framework of the PDE-Net 2.0 and then introduce the details on these two components. Regularization and training strategies will be given near the end of this section.

\subsection{Architecture of PDE-Net 2.0}\label{SubSec:PDE-Net}

Inspired by the dynamic system perspective of deep neural networks \cite{chen2015learning,weinan2017proposal,haber2017stable,sonoda2017double,Lu2018Beyond,chang2017multi}, we consider forward Euler as the temporal discretization of the evolution PDE \eqref{E:PDE:General}, and unroll the discrete dynamics to a feed-forward network. One may consider more sophisticated temporal discretization which naturally leads to different network architectures \cite{Lu2018Beyond}. For simplicity, we focus on forward Euler in this paper.

\subsubsection{\textbf{\texorpdfstring{$\delta t$}{delta-t}-block:}}
Let $\tilde{U}(t+\delta t,\cdot)$ be the predicted value at time $t+\delta t$ based on $\tilde{U}(t,\cdot)$. Then, we design an approximation framework as follows
\begin{equation}\label{m7}
    \tilde{U}(t+\delta t, \cdot)\approx \tilde{U}(t,\cdot)+\delta t \cdot \symnet_m^k(D_{00}\tilde{U}, D_{01}\tilde{U}, D_{10}\tilde{U}, \cdots).
\end{equation}
Here, the operators $D_{ij}$ are convolution operators with the underlying filters denoted by $q_{ij}$, i.e. $D_{ij}u=q_{ij}\circledast u$. The operators $D_{10}$, $D_{01}$, $D_{11}$, etc. approximate differential operators, i.e. $D_{ij}u\approx\frac{\partial^{i+j}u}{\partial^ix\partial^jy}$. In particular, the operators $D_{00}$ is a certain averaging operator. The purpose of introducing the average operators in stead of simply using the identity is to improve the expressive power of the network and enables it to capture more complex dynamics. 

Other than the assumption that the observed dynamics is governed by a PDE of the form \eqref{E:PDE:General}, we assume that the highest order of the PDE is less than some positive integer. Then, we can assume that $F$ is a function of $m$ variables with known $m$. The task of approximating $F$ in \eqref{E:PDE:General} is equivalent to a multivariate regression problem. In order to be able to identify the analytic form of $F$, we use a symbolic neural network denote by $\symnet_m^k$ to approximate $F$, where $k$ denotes the depth of the network. Note that, if $F$ is a vector function, we use multiple $\symnet_m^k$ to approximate the components of $F$ separately. 

Combining the aforementioned approximation of differential operators and the nonlinear response function, we obtain an approximation framework \eqref{m7} which will be referred to as a $\delta t$-block~ (see Figure \ref{Figure:dt:block}). Details of these two components can be found later in Section \ref{SubSec:Conv:and:Diff} and Section \ref{SubSec:SymbolicSubNet}.

\begin{figure}[ht]
\begin{center}
\includegraphics[width=0.8\linewidth]{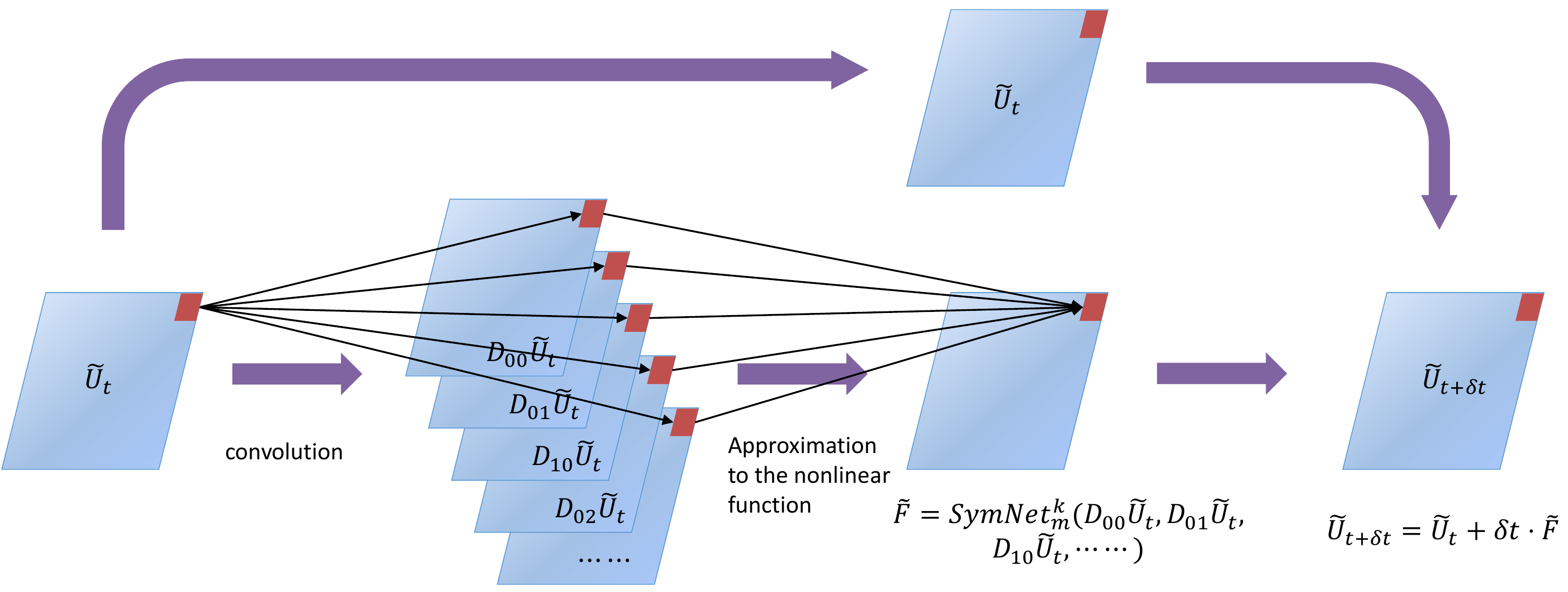}
\end{center}
\caption{The schematic diagram of a $\delta t$-block.}\label{Figure:dt:block}
\label{fig1}
\end{figure}

\subsubsection{\textbf{Multiple \texorpdfstring{$\delta t$}{delta-t}-Blocks:}}

One $\delta t$-block only guarantees the accuracy of one-step dynamics, which does not take error accumulation into consideration. In order to facilitate a long-term prediction, we stack multiple $\delta t$-blocks into a deep network, and call this network the \textit{PDE-Net 2.0} (see Figure \ref{Figure:multi:dt:block}). The importance of stacking multiple $\delta t$-blocks will be demonstrated by our numerical experiments in Section \ref{Sec:Results:Burgers} and \ref{Sec:Results:Heat}.

The PDE-Net 2.0 can be easily described as: (1) stacking one $\delta t$-block multiple times; (2) sharing parameters in all $\delta t$-blocks. Given a set of observed data $\{U(t,\cdot)\}$, training a PDE-Net 2.0 with $n$ $\delta t$-blocks needs to minimize the accumulated error $||U(t+n\delta t,\cdot)-\tilde{U}(t+n\delta t,\cdot)||_2^2$, where $\tilde{U}(t+n\delta t,\cdot)$ is the output from the PDE-Net 2.0 (i.e. $n$ $\delta t$-blocks) with input $U(t,\cdot)$, and $U(t+n\delta t,\cdot)$ is observed training data.
\begin{figure}[ht]
\begin{center}
\includegraphics[width=0.8\linewidth]{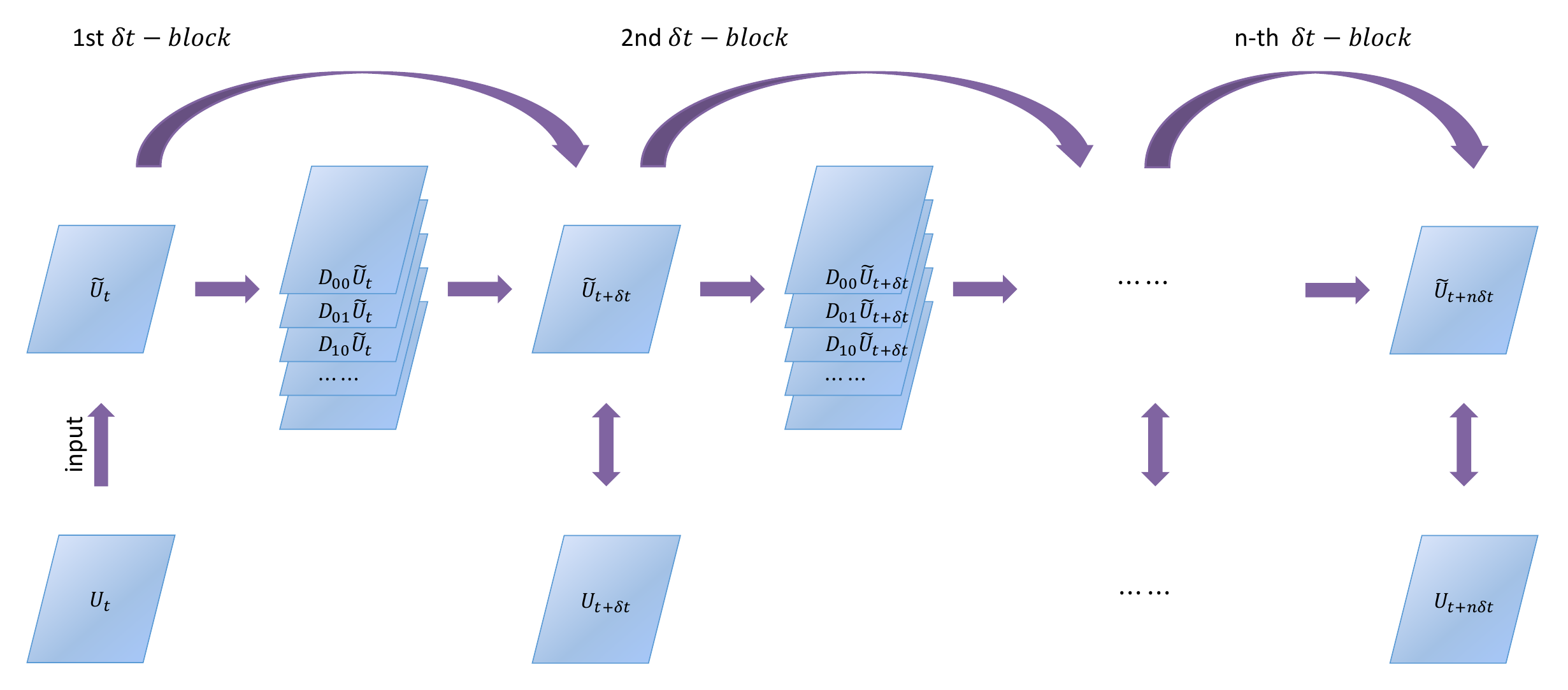}
\end{center}
\caption{The schematic diagram of the PDE-Net 2.0.}\label{Figure:multi:dt:block}
\end{figure}

\subsection{Convolutions and Differentiations}\label{SubSec:Conv:and:Diff}

In the original PDE-Net \cite{Long2018PDE}, the learnable filters are properly constrained so that we can easily identify their correspondence to differential operators. PDE-Net 2.0 adopts the same constrains as the original version of the PDE-Net. For completeness, we shall review the related notions and concepts, and provide more details.

A profound relationship between convolutions and differentiations was presented by \cite{cai2012image,dong2017image}, where the authors discussed the connection between the order of sum rules of filters and the orders of differential operators. Note that the definition of convolution we use follows the convention of deep learning, which is defined as
$$(f\circledast q)[l_1,l_2]:=\sum_{k_1,k_2}q[k_1,k_2]f[l_1+k_1,l_2+k_2].$$
This is essentially correlation instead of convolution in the mathematics convention. Note that if $f$ is of finite size, we use periodic boundary condition.

The order of sum rules is closely related to the order of vanishing moments in wavelet theory ~\cite{daubechies1992ten,mallat1999wavelet}. We first recall the definition of the order of sum rules.
\begin{definition}[Order of Sum Rules]
For a filter $q$, we say $q$ to have sum rules of order $\alpha=(\alpha_1,\alpha_2)$, where $\alpha \in \mathbb{Z}^2_+$, provided that
\begin{small}
\begin{equation}\label{m1}
    \sum_{k\in \mathbb{Z}^2}k^\beta q[k]=0
\end{equation}
\end{small}
for all $\beta=(\beta_1,\beta_2) \in \mathbb{Z}^2_+$ with $|\beta|:=\beta_1+\beta_2<|\alpha|$ and for all $\beta \in \mathbb{Z}^2_+$ with $|\beta|=|\alpha|$ but $\beta \neq \alpha$. If (\ref{m1}) holds for all $\beta \in \mathbb{Z}^2_+$ with $|\beta|<K$ except for $\beta \neq \bar{\beta}$ with certain $\bar{\beta} \in \mathbb{Z}^2_+$ and $|\bar{\beta}|=J<K$, then we say $q$ to have total sum rules of order $K\backslash \{J+1\}$.
\end{definition}

In practical implementation, the filters are normally finite and can be understood as matrices. For an $N \times N$ filter $q$ ($N$ being an odd number), assuming the indices of $q$ start from $-\frac{N-1}{2}$, \eqref{m1} can be written in the following simpler form
\begin{equation*}
    \sum_{l=-\frac{N-1}{2}}^{\frac{N-1}{2}}\sum_{m=-\frac{N-1}{2}}^{\frac{N-1}{2}}l^{\beta_1}m^{\beta_2} q[l,m]=0.
\end{equation*}
The following proposition from \cite{dong2017image} links the orders of sum rules with orders of differential operators.

\begin{proposition}\label{m2}
  Let $q$ be a filter with sum rules of order $\alpha \in \mathbb{Z}^2_+$.
  Then for a smooth function $F(x)$ on $\mathbb{R}^2$, we have
  \begin{equation}
    \frac{1}{\varepsilon^{|\alpha|}}\sum_{k\in \mathbb{Z}^2}q[k]F(x+\varepsilon k)=C_\alpha \frac{\partial^\alpha}{\partial x^\alpha}F(x)+O(\varepsilon), \mbox{as}~\varepsilon \rightarrow 0.
  \end{equation}
If, in addition, $q$ has total sum rules of order $K\backslash \{|\alpha|+1\}$ for some $K>|\alpha|$, then
\begin{equation}\label{m3}
    \frac{1}{\varepsilon^{|\alpha|}}\sum_{k\in \mathbb{Z}^2}q[k]F(x+\varepsilon k)=C_\alpha \frac{\partial^\alpha}{\partial x^\alpha}F(x)+O(\varepsilon^{K-|\alpha|}), \mbox{as}~\varepsilon \rightarrow 0.
\end{equation}

\end{proposition}

According to Proposition~\ref{m2}, an $\alpha$th order differential operator can be approximated by the convolution of a filter with $\alpha$ order of sum rules. Furthermore, according to \eqref{m3}, one can obtain a high order approximation of a given differential operator if the corresponding filter has an order of total sum rules with $K>|\alpha|+k, k \geqslant 1$. For example, consider filter
\[
q=
\left(
\begin{array}{rrr}
1&0&-1\\
2&0&-2\\
1&0&-1
\end{array}
\right),
\]
It has a sum rules of order $(1,0)$, and a total sum rules of order $3\backslash\{2\}$. Thus, up to a constant and a proper scaling, $q$ corresponds to a discretization of $\frac{\partial}{\partial x}$ with second order accuracy.

For an $N\times N$ filter $q$, define the \textit{moment matrix} of $q$ as
\begin{equation}\label{mm1}
M(q)=(m_{i,j})_{N\times N},
\end{equation}
where
\begin{equation}\label{E:q2m}
m_{i,j}=\frac{1}{i!j!}\sum_{k_1,k_2=-\frac{N-1}{2}}^{\frac{N-1}{2}}k_1^{i}k_2^{j}q[k_1,k_2],i,j=0,1,\ldots,N-1.
\end{equation}
We shall call the $(i, j)$-element of $M(q)$ the $(i, j)$-moment of $q$ for simplicity.
For any smooth function $f:\mr^2\to\mr$, we apply convolution on the sampled version of $f$ with respect to the filter $q$. By Taylor's expansion, one can easily obtain the following formula
\begin{eqnarray}\label{E:derivemomentmatrix}
&&\sum_{k_1,k_2=-\frac{N-1}{2}}^{\frac{N-1}{2}}q[k_1,k_2]f(x+k_1\delta x,y+k_2\delta y) \nonumber \\
&=&\sum_{k_1,k_2=-\frac{N-1}{2}}^{\frac{N-1}{2}}q[k_1,k_2]\sum_{i,j=0}^{N-1}\frac{\pd^{i+j}f}{\pd^ix\pd^jy}\bigg|_{(x,y)}\frac{k_1^ik_2^j}{i!j!}\delta x^i\delta y^j+o(|\delta x|^{N-1}+|\delta y|^{N-1})\nonumber\\
&=&\sum_{i,j=0}^{N-1}m_{i,j}\delta x^i\delta y^j\cdot\frac{\pd^{i+j}f}{\pd^ix\pd^jy}\bigg|_{(x,y)}+o(|\delta x|^{N-1}+|\delta y|^{N-1}).
\end{eqnarray}
From (\ref{E:derivemomentmatrix}) we can see that filter $q$ can be designed to approximate any differential operator with prescribed order of accuracy by imposing constraints on $M(q)$.

For example, if we want to approximate $\frac{\partial u}{\partial x}$ (up to a constant) by convolution $q\circledast u$ where $q$ is a $3\times 3$ filter, we can consider the following constrains on $M(q)$:
\begin{equation}\label{E:constraints:example}
  \left(\begin{array}{ccc}
    0&0&\star \\
    1&\star&\star\\
    \star&\star&\star
\end{array}
\right)\quad\mbox{or}\quad
\left(\begin{array}{ccc}
    0&0&0\\
    1&0&\star\\
    0&\star&\star
\end{array}
\right).
\end{equation}
Here, $\star$ means no constraint on the corresponding entry. The constraints described by the moment matrix on the left of \eqref{E:constraints:example} guarantee the approximation accuracy is at least first order, and the one on the right guarantees an approximation of at least second order. In particular, when all entries of $M(q)$ are constrained, e.g.
$$
M(q)=\left(\begin{array}{ccc}
    0&0&0\\
    1&0&0\\
    0&0&0
\end{array}
\right),$$
the corresponding filter can be uniquely determined. In the PDE-Net 2.0, all filters are learned subjected to partial constraints on their associated moment matrices, with at least second order accuracy.

\subsection{Design of \texorpdfstring{$\symnet$}{SymNet}: a Symbolic Neural Network}\label{SubSec:SymbolicSubNet}

The symbolic neural network $\symnet_m^k$ of the PDE-Net 2.0 is introduced to approximate the multivariate nonlinear response function $F$ of \eqref{E:PDE:General}. Neural networks have recently been proven effective in approximating multivariate functions in various scenarios \cite{poggio2017and,shaham2016provable,montanelli2017deep,wang2018exponential,he2018relu}. For the problem we have in hand, we not only require the network to have good expressive power, but also good transparency so that the analytic form of $F$ can be readily inferred after training. Our design of $\symnet_m^k$ is motivated by EQL/EQL$^\div$ proposed by \cite{martius2016extrapolation,sahoo2018learning}.

The $\symnet_m^k$, as illustrated in Figure \ref{Figure:symnet}, is a network that takes an $m$ dimensional vector as input and has $k$ hidden layers.
\begin{figure}[ht]
\centering
\includegraphics[width=0.5\linewidth]{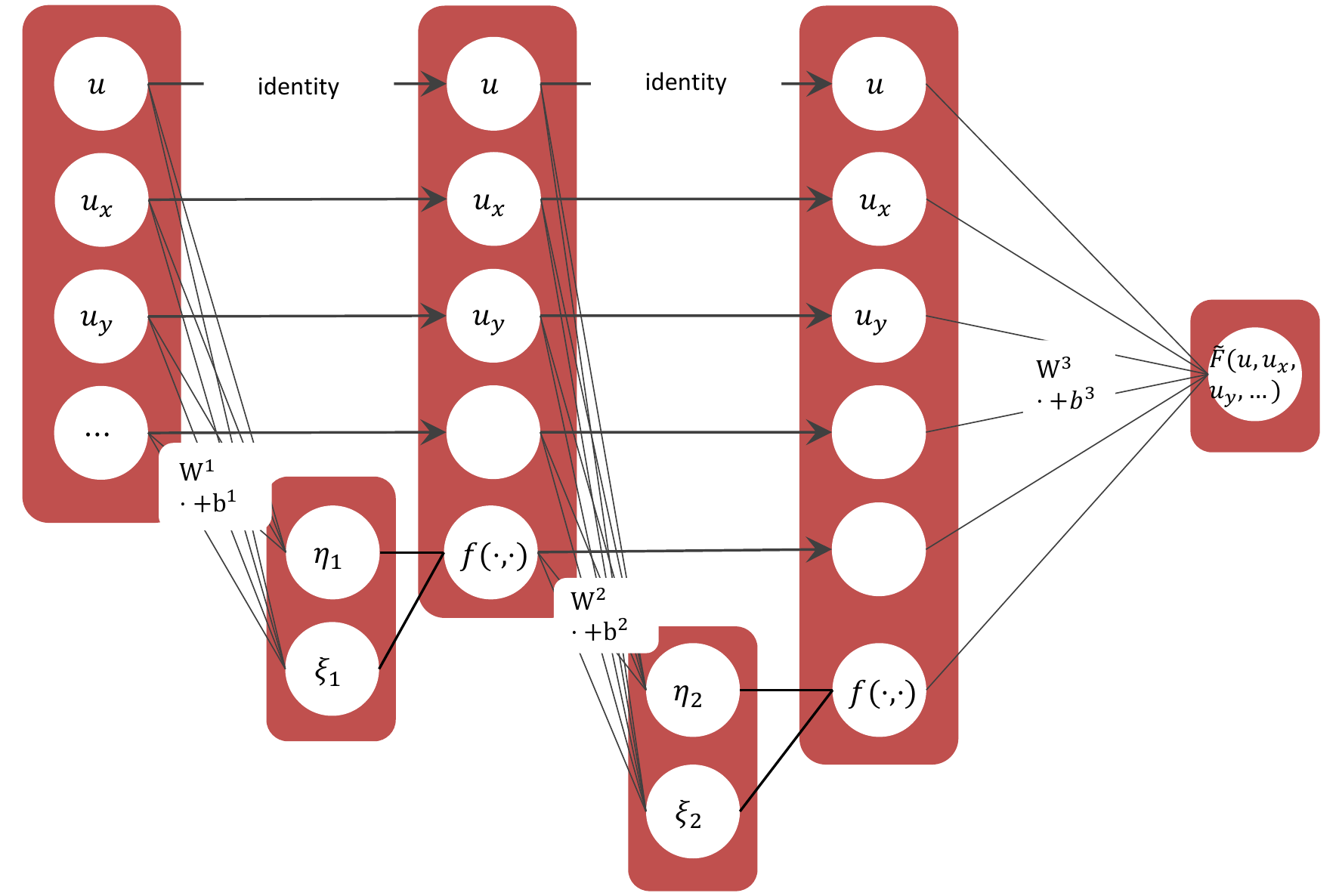}
\caption{The schematic diagram of $\symnet$}\label{Figure:symnet}
\end{figure}
Figure \ref{Figure:symnet} shows the symbolic neural network with two hidden layers, i.e. $\symnet_m^2$, where $f$ is a dyadic operation unit, e.g. multiplication or division. In this paper, we focus on multiplication, i.e. we take $f(a,b)=a\times b$. Different from EQL/EQL$^\div$, each hidden layer of the $\symnet_m^k$ directly takes the outputs from the preceding layer as inputs, rather than a linear combination of them. Furthermore, it adds one additional variable (i.e. $f(\cdot,\cdot)$) at each hidden layer. To better understand $\symnet_m^k$, we present an example in Algorithm \ref{Alg:symnet62} showing how $\symnet_6^2$ is constructed. In particular, when $b^i=0, i=1,2,3$, 
\begin{equation*}
W^1=\left(\begin{array}{cccccc}
1 & 0 & 0 & 0 & 0 & 0 \\
0 & 1 & 0 & 0 & 0 & 0
\end{array}\right),
W^2=\left(\begin{array}{ccccccc}
0 & 0 & 1 & 0 & 0 & 0 & 0\\
0 & 0 & 0 & 1 & 0 & 0 & 0
\end{array}\right)
\end{equation*}
and $W^3=(0,0,0,0,0,0,-1,-1)$, then $\symnet_6^2(u,u_x,u_y,v,v_x,v_y)=-uu_x-vu_y$ which is the right-hand-side of $\frac{\partial u}{\partial t}$ of the Burgers' equation without viscosity.

\begin{algorithm}[ht]\label{symnet:exp}
\caption{$\symnet_6^2$}\label{Alg:symnet62}
\textbf{Input:} $u,u_x,u_y,v,v_x,v_y\in\mr$,
\begin{algorithmic}
\STATE $(\eta_1,\xi_1)^\top = W^1\cdot(u,u_x,u_y,v,v_x,v_y)^\top+b^1,W^1\in\mr^{2\times6},b^1\in\mr^2$;
\STATE $f_1 = \eta_1\cdot\xi_1$;
\STATE $(\eta_2,\xi_2)^\top = W^2\cdot(u,u_x,u_y,v,v_x,v_y,f_1)^\top+b^2,W^2\in\mr^{2\times7},b^2\in\mr^2$;
\STATE $f_2=\eta_2\cdot\xi_2$;
\end{algorithmic}
\textbf{Output:} $F=W^3\cdot(u,u_x,u_y,v,v_x,v_y,f_1,f_2)^\top+b^3\in\mr,W^3\in\mr^{1\times8},b^3\in\mr$.
\end{algorithm}

The $\symnet_m^k$ can represent all polynomials of variables $(x_1,x_2,\ldots,x_m)$ with the total number of multiplications not exceeding $k$. If needed, one can add more operations to the $\symnet_m^k$ to increase the capacity of the network.

Now, we show that $\symnet_m^k$ is more compact than the dictionaries of SINDy. For that, we first introduce some notions.
\begin{definition}
Define the set of all polynomials of $m$ variables $(x_1,\cdots,x_m)$ with the total number of multiplications not exceeding $k$ as $\kpoly^k[x_1,\cdots,x_m]$. Here, the total number of multiplications of $\kpoly^k[x_1,\cdots,x_m]$ is counted as follows:
\begin{itemize}
\item For any monomial of degree $k$, if $k\geq2$, then the number of multiplications of the monomial is counted as $k-1$. When $k=1$ or $0$, the count is 0.
\item For any polynomial $P$, the total number of multiplications is counted as the sum of the number of multiplications of its monomials.
\end{itemize}
\end{definition}
For example, $\sum_{i=1}^mx_i+\sum_{i=1}^{k}x_ix_{i+1}$ and $\prod_{i=1}^{k+1}x_i$ with $k<m$ are all members of $\kpoly^k[x_1,\cdots,x_m]$. The elements in $\kpoly^k[x_1,\cdots,x_m]$ are of simple forms when $k$ is relatively small. The following proposition shows that $\symnet_m^k$ can represent all polynomials of variables $(x_1,x_2,\ldots,x_m)$ with the total number of multiplications not exceeding $k$. Note that the actual capacity of $\symnet_m^k$ is larger than $\kpoly^k[x_1,\cdots,x_m]$, i.e. $\kpoly^k[x_1,\cdots,x_m]$ is a subset of the set of functions that $\symnet_m^k$ can represent.

\begin{proposition}
For any $P\in\kpoly^k[x_1,\cdots,x_m]$, there exists a set of parameters for $\symnet_m^k$ such that
\[P=\symnet_m^k(x_1,\cdots,x_m).\]
\end{proposition}

\begin{proof}
We prove this proposition by induction. When $k=1$, the conclusion obviously holds. Suppose the conclusion holds for $k$. For any polynomial $P\in\kpoly^{k+1}[x_1,\cdots,x_m]$, we only need to consider the cases when $P$ has a total number of multiplications greater than 1.

We take any monomial of $P$ that has degree greater than 1, which we suppose take the form $x_1x_2\cdot A$ where $A$ is a monomial of variable $(x_1,\ldots,x_m)$. Then, $P$ can be written as $P=x_1x_2 A+Q$. Define new variable $x_{m+1}=x_1x_2$. Then, we have \[P=x_{m+1}A+Q\in\kpoly^{k}[x_1,\cdots,x_m,x_{m+1}].\] By the induction hypothesis, there exists a set of parameters such that $P=\symnet_{m+1}^k(x_1,\cdots,x_{m+1})$.

We take the linear transform between the input layer and the first hidden layer of $\symnet_m^{k+1}$ as \[W^1\cdot(x_1,\cdots,x_m)^\top+b_1=(x_1,x_2)^\top.\] Then, the output of the first hidden layer is $x_1,x_2,\cdots,x_m,x_1x_2$. If we use it as the input of $\symnet_{m+1}^{k}$, we have
\begin{eqnarray*}
P(x_1,\cdots,x_m)&=&\symnet_{m+1}^k(x_1,\cdots,x_m,x_1x_2)\\
                    &=&\symnet_m^{k+1}(x_1,\cdots,x_m),
\end{eqnarray*}
which concludes the proof.
\end{proof}

SINDy constructs a dictionary that incudes all possible monomials up to a certain degree. Observe that there are totally $\binom{m+l}{l}$ monomials with $m$ variables and a degree not exceeding $l (\in\mn)$. Our symbolic network, however, is more compact than SINDy. The following proposition compares the complexity of $\symnet_m^k$ and SINDy, whose proof is straightforward.
\begin{proposition}
Let $P\in\kpoly^k[x_1,\cdots,x_m]$ and suppose $P$ have monomials of degree $\le l$.
\begin{itemize}
\item The memory load of $\symnet_m^k$ that approximates $P$ is $O(m+k)$. The number of flops for evaluating $\symnet_m^k$ is $O(k(m+k))$.
\item Constructing a dictionary with all possible polynomials of degree $l$ requires a memory load of $\binom{m+l}{l}$, and evaluation of a linear combination of dictionary members requires $O(\binom{m+l}{l})$ flops.
\end{itemize}
\end{proposition}
We use the following example to show the advantage of $\symnet$ over SINDy. Consider two variables $u,v$ and all of their derivatives of order $\le 2$:  $$(u,v,u_x,u_y,\cdots,v_{yy})\in\mathbb{R}^{12}.$$ Suppose the polynomial to be approximated is $P=-uu_x-vu_y+u_{xx}$. For $k=l=3$, the size of the dictionary of SINDy is $\binom{15}{3}=455$ and the computation of linear combination of the elements requires $909$ flops. The memory load of $\symnet_{12}^3$, however, is 15 and an evaluation of the network requires 180 flops. Therefore, $\symnet$ can significantly reduce memory load and computation cost when input data is large. Note that when $k$ is large and $l$ small, $\symnet_m^k$ is worse than SINDy. However, for system identification problems, we normally wish to obtain a compact representation (i.e. smaller $k$). Thus, $\symnet$ takes full advantage of this prior knowledge and can significantly save on memory and computation cost which is crucial in the training of the PDE-Net 2.0.

\subsection{Loss Function and Regularization}

We adopt the following loss function to train the proposed PDE-Net 2.0:
\[L=L^{data}+\lambda_1L^{moment}+\lambda_2L^{\symnet},\]
where the hyper-parameters $\lambda_1$ and $\lambda_2$ are chosen as $\lambda_1=0.001$ and $\lambda_2=0.005$. Now, we present details on each of the term of the loss function and introduce pseudo-upwind as an additional constraint on PDE-Net 2.0.

\subsubsection{Data Approximation \texorpdfstring{$L^{data}$}{L-data}}

Consider the data set $\{U_j(t_i,\cdot): 1\le i\le n, 1\le j\le N\}$, where $n$ is the number of $\delta t$-blocks and $N$ is the total number of samples. The index $j$ indicates the $j$-th solution path with a certain initial condition of the unknown dynamics. Note that one can split a long solution path into multiple shorter ones. We would like to train the PDE-Net 2.0 with $n$ $\delta t$-blocks. For a given $n\ge1$, every pair of the data $\{U_j(t_0,\cdot), U_j(t_i,\cdot)\}$, for each $j$ and $i\leq n$, is a training sample, where $U_j(t_0,\cdot)$ is the input and $U_j(t_{i},\cdot)$ is the label that we need to match with the output from the network. For that, we define the data approximation term $L^{data}$ as:
\begin{equation*}
L^{data}=\frac1{n}\sum_{i=1}^n\sum_{j=1}^N l_{ij}/\delta t^2, \mbox{where} \quad l_{ij}=||U_j(t_{i},\cdot)-\tilde{U}_j(t_{i},\cdot)||_2^2,
\end{equation*}
where $\tilde{U}_j(t_{i},\cdot)$ is the output of the PDE-Net 2.0 with $\tilde{U}_j(t_0,\cdot)=U_j(t_0,\cdot)$ as the input.

\subsubsection{Regularization: \texorpdfstring{$L^{moment}$}{L-moment} and \texorpdfstring{$L^{\symnet}$}{L-SymNet}}

For a given threshold $s$, defined Huber's loss function $\ell_1^s$ as
\[\ell_1^s(x)=\left\{
\begin{array}{ll}
|x|-\frac{s}{2} & \text{if }|x|>s \\
\frac{1}{2s}x^2 & \text{else.}
\end{array}\right.
\]
Then, we define $L^{moment}$ as
\[L^{moment}=\sum_{i,j}\sum_{i_1,j_1} \ell_1^s(M(q_{ij})[i_1,j_1]),\] where ${q_{ij}}$ are the filters of PDE-Net 2.0 and $M(q)$ is the moment matrix of $q$. We use this loss function to regularize the learnable filters to reduce overfitting. In our numerical experiments, we will use $s=0.01$.

Given $\ell_1^s$ as defined above, we use it to enforce sparsity on the parameters of $\symnet$. This will help to reduce overfitting and enable more stable prediction. The loss function $L^{\symnet}$ is defined as
\[L^{\symnet}=\sum_{p\in\text{parameters of }\symnet} \ell_1^s(p).\]
We set $s=0.001$ in our numerical experiments.

\subsubsection{Pseudo-upwind}
In numerical PDEs, to ensure stability of a numerical scheme, we need to design conservation schemes or use upwind schemes \cite{liu1994weighted,shu1998essentially,leveque2002finite}. This is also important for PDE-Net 2.0 during inferencing. However, the challenge we face is that we do not know apriori the form or the type of the PDE. Therefore, we introduce a method called \textit{pseudo-upwind} to help with maintaining stability of the PDE-Net 2.0.

Given a 2D filter $q$, define the flipping operators $\text{flip}_x(q)$ and $\text{flip}_y(q)$ as
\[\left(\text{flip}_x(q)\right)[k_1,k_2] = -q[-k_1,k_2],\quad k_1,k_2=-(N-1)/2,\cdots,(N-1)/2\]
and
\[\left(\text{flip}_y(q)\right)[k_1,k_2] = -q[k_1,-k_2],\quad k_1,k_2=-(N-1)/2,\cdots,(N-1)/2.\]
In each $\delta t$-block of the PDE-Net 2.0, before we apply convolution with a filter, we first use $\symnet$ to determine whether we should use the filter or flip it first. We use the following univariate PDE as an example to demonstrate our idea. Given PDE $u_t=F(u,\cdots)$, suppose the input of a $\delta t$-block is $u$. The algorithm of pseudo-upwind is described by the following Algorithm \ref{Alg:pseudo-upwind}.


\begin{algorithm}[ht]
\caption{pseudo-upwind in $\delta t$-block}\label{Alg:pseudo-upwind}
\textbf{Input:} $u$
\begin{algorithmic}
\STATE $u_{01}=q_{01}\circledast u,u_{10}=q_{10}\circledast u$
\STATE $\tilde{F}_{0}=\symnet_6^k(u,u_{01},u_{10},q_{02}\circledast u,q_{11}\circledast u,q_{20}\circledast u)$
\IF {$\frac{\pd \tilde{F}_{0}}{\pd u_{01}}>0$}
    \STATE $u_x = u_{01}$
\ELSE
    \STATE $u_x = \text{flip}_x(q_{01})\circledast u$
\ENDIF
\IF {$\frac{\pd \tilde{F}_{0}}{\pd u_{10}}>0$}
    \STATE $u_y = u_{10}$
\ELSE
    \STATE $u_y = \text{flip}_y(q_{10})\circledast u$
\ENDIF
\STATE $\tilde{F}=\symnet_6^k(u,u_x,u_y,q_{02}\circledast u,q_{11}\circledast u,q_{20}\circledast u)$
\end{algorithmic}
\textbf{Return:} $u+\delta t\tilde{F}$
\end{algorithm}

\begin{remark}
Note that the algorithm does not exactly enforce upwind in general. This is why we call it pseudo-upwind. We further note that:
\begin{itemize}
\item Given a PDE of the form $u_t=G(u)u_x+H(u)u_y+\lambda(u_{xx}+u_{yy})$, we can use $G(u)$ and $H(u)$ to determine whether we should flip a filter or not.
\item For a vector PDE, such as $U=(u,v)^\top$, we can use, for instance, $\frac{\pd\tilde{F}_{0}}{\pd u_{01}},\frac{\pd\tilde{F}_{0}}{\pd v_{10}}$ to determine how we should approximate $u_x$ and $u_y$ in the $\delta t$-block \cite{shu1998essentially}.
\end{itemize}
\end{remark}

\subsection{Initialization and training}
In the PDE-Net 2.0, parameters can be divided into three groups: 1) moment matrices to generate convolution kernels; 2) the parameters of $\symnet$; and 3) hyper-parameters, such as the number of filters, the size of filters, the number of $\delta t$-Blocks and number of hidden layers of $\symnet$, regularization weights $\lambda_1,\lambda_2,\lambda_3$, etc. The parameters of the $\symnet$ are shared across the computation domain $\Omega$, and are initialized by random sampling from a Gaussian distribution. For the filters, we initialize $D_{01},D_{10}$ with second order pseudo-upwind scheme and central difference for all other filters. For example, if the size of the filters were set to be $5\times5$, then the initial values of the convolution kernels $q_{01},q_{02}\in\mr^{5\times5}$ are
\[
  q_{01}=\left(\begin{array}{ccccc}
      &   & \bm{0} &   &   \\
    0 & 0 &-3 & 4 &-1 \\ 
      &   & \bm{0} &   &   
  \end{array}
  \right),\quad
  q_{02}=\left(\begin{array}{ccccc}
      &   & \bm{0} &   &   \\
    0 & 1 &-2 & 1 & 0 \\ 
      &   & \bm{0} &   &   
  \end{array}
  \right).
\]

We use layer-wise training to train the PDE-Net 2.0. We start with training the PDE-Net 2.0 on the first $\delta$t-block with a batch of data, and
then use the results of the first $\delta$t-block as the initialization and restart training on the first two $\delta$t-blocks with another batch. Repeat this procedure until we complete all $n$ $\delta$t-blocks. Note that all the parameters in each of the $\delta$t-block are shared across layers. In addition, we add a warm-up step before the training of the first $\delta$t-block by fixing filters and setting regularization term to be 0 (i.e. $\lambda_1=\lambda_2=0$). The warm-up step is to obtain a good initial guess of the parameters of $\symnet$.

To demonstrate the necessity of having learnable filters, we will compare the PDE-Net 2.0 containing learnable filters with the PDE-Net 2.0 having fixed filters. To differentiate the two cases, we shall call the PDE-Net 2.0 with fixed filters the ``Frozen-PDE-Net 2.0''. Note that for Frozen-PDE-Net 2.0, the filters are fixed to be the initial values we choose to train the regular PDE-Net 2.0. This is a natural choice since when we know apriori that the PDE is Burgers' equation, it would be a stable finite difference scheme. However, intuitively speaking, freezing any finite difference approximations of the differential operators during training of PDE-Net 2.0 is not ideal, because you cannot possibly know which numerical scheme to use without knowing the form of the PDE. Therefore, for inverse problem, it is better to learn both the PDE model and the discretization of the PDE model simultaneously. This assertion is supported by our emperical comparisons between frozen and regular PDE-Net 2.0 in Table \ref{Tab:Analytic:Burgers} and \ref{Tab:Analytic:CDR}.

\section{Numerical Studies: Burgers' Equation}\label{Sec:Results:Burgers}
Burgers' equation is a fundamental partial differential equation in many areas such as fluid mechanics and traffic flow modeling. It has a lot in common with the Navier-Stokes equation, e.g. the same type of advective nonlinearity and the presence of viscosity.
\subsection{Simulated data, training and testing}
In this section we consider a 2-dimensional Burger's equation with periodic boundary condition on $\Omega=[0,2\pi]\times [0,2\pi]$,
\begin{equation}\label{E:testPDE:Burgers}
\left\{
\begin{array}{ll}
\frac{\partial U}{\partial t}&=-U\cdot\nabla U+\nu\Delta U,U=(u,v)^\top\\
U|_{t=0}&=U_0(x,y),
\end{array}
\right.
\end{equation}
with $(t,x,y)\in [0, 4]\times\Omega,$
where $\nu=0.05$

The training data is generated by a finite difference scheme on a $128\times 128$ mesh and then restricted to a $32 \times 32$ mesh. The temporal discretization is 2nd order Runge-Kutta with time step $\delta t = \frac{1}{1600}$, the spatial discretization uses a 2nd order upwind scheme for $\nabla$ and the central difference scheme for $\Delta$. The initial value $u_0(x,y),v_0(x,y)$ takes the following form,
\begin{equation}\label{E:random:initialize}
  w(x,y) = \frac{2w_{0}(x,y)}{\max_{x,y} |w_{0}|}+c
\end{equation}
where $w_{0}(x,y)=\sum_{|k|,|l|\leq 4}\lambda_{k,l}\cos(kx+ly)+\gamma_{k,l}\sin(kx+ly)$, $\lambda_{k,l},\gamma_{k,l}\sim\mathcal{N}(0,1),c\sim \mathcal{U}(-2,2)$. Here, $\mathcal{N}(0,1),\mathcal{U}(-2,2)$ represents the standard normal distribution and uniform distribution on $[-2,2]$ respectively. We also add noise to the generated data:
\begin{equation}\label{E:add:noise}
  \widehat{U}(t,x,y)=U(t,x,y)+0.001\times MW
\end{equation}
where $M=\max_{x,y,t}\{U(t,x,y)\}$, $W \sim \mathcal{N}(0,1)$.

Suppose we know a priori that the order of the underlying PDE is no more than 2, we can use two $\symnet_{12}^5$ to approximate the right-hand-side nonlinear response function of \eqref{E:testPDE:Burgers} component-wise. Let $U=(u,v)^\top$. We denote the two $\symnet_{12}^5$ as $Net_u$ and $Net_v$ respectively. Then, each $\delta t$-block of the PDE-Net 2.0 can be written as
\begin{eqnarray*}
\tilde{u}(t_{i+1},\cdot)&=&\tilde{u}(t_i,\cdot)+\delta t\cdot Net_u(D_{00}\tilde{u},D_{01}\tilde{u},\cdots,D_{20}\tilde{u},
D_{00}\tilde{v},\cdots,D_{20}\tilde{v}) \\
\tilde{v}(t_{i+1},\cdot)&=&\tilde{v}(t_i,\cdot)+\delta t\cdot Net_v(D_{00}\tilde{u},D_{01}\tilde{u},\cdots,D_{20}\tilde{u},
D_{00}\tilde{v},\cdots,D_{20}\tilde{v})
\end{eqnarray*}
where $\{D_{ij}:0\le i+j\leq 2\}$ are convolution operators.

During training and testing, the data is generated on-the-fly. The size of the filters that will be used is $5\times 5$. The total number of parameters in $Net_u$ and $Net_v$ is 336, and the number of trainable parameters in moment matrices is 105 for $5\times5$ filters ($6\times5\times5-45$ (constraint on moment)). During training, we use BFGS, instead of SGD, to optimize the parameters. We use 28 data samples per batch to train each $\delta t$-block and we only construct the PDE-Net 2.0 up to 9 layers, which requires totally 420 data samples during the whole training procedure.

\subsection{Results and discussions}

We first demonstrate the ability of the trained PDE-Net 2.0 to recover the analytic form of the unknown PDE model. We use the symbolic math tool in python to obtain the analytic form of $\symnet$. Results are summarized in Table \ref{Tab:Analytic:Burgers}. As one can see from Table \ref{Tab:Analytic:Burgers} that we can recover the terms of the Burgers' equation with good accuracy, and using learnable filters helps with the identification of the PDE model. Furthermore, the terms that are not included in the Burgers' equation all have relatively small weights in the $\symnet$ (see Figure \ref{Fig:Img:remainder}). 

\begin{table}[htp]
  \centering
  \caption{PDE model identification.}\label{Tab:Analytic:Burgers}
  \begin{tabular}{|c|l|}
    \hline
    \multirow{2}{*}{Correct PDE} & $u_t=-uu_x-vu_y+0.05(u_{xx}+u_{yy})$ \\
    & $v_t=-uv_x-vv_y+0.05(v_{xx}+v_{yy})$ \\
    \hline
    \multirow{2}{*}{Frozen-PDE-Net 2.0} & $u_t=-{\color{blue}0.906}uu_{x}-{\color{blue}0.901}vu_{y}+0.033u_{xx}+0.037u_{yy}$ \\
    & $v_t=-{\color{blue}0.907}vv_{y}-{\color{blue}0.902}uv_{x}+0.039v_{xx}+0.032v_{yy}$\\
    \hline
    \multirow{2}{*}{PDE-Net 2.0} & $u_t=-{\color{orange}0.979}uu_{x}-{\color{orange}0.973}u_{y}v+0.052u_{xx}+0.051u_{yy}$ \\
    & $v_t=-{\color{orange}0.973}uv_{x}-{\color{orange}0.977}vv_{y}+0.053v_{xx}+0.051v_{yy}$ \\
    \hline
  \end{tabular}
\end{table}

We also demonstrate the ability of the trained PDE-Net 2.0 in prediction, i.e. the ability to generalize. After the PDE-Net 2.0 with $n$ $\delta t$-blocks ($1\le n\le 9$) is trained, we randomly generate 1000 initial guesses based on \eqref{E:random:initialize} and \eqref{E:add:noise}, feed them to the PDE-Net 2.0, and measure the relative error between the predicted dynamics (i.e. the output of the PDE-Net 2.0) and the actual dynamics (obtained by solving \eqref{E:testPDE:Burgers} using high precision numerical scheme). The relative error between the true data $U$ and the predicted data $\tilde U$ is defined as
$\epsilon=\frac{\|\tilde{U}-U\|_2^2}{\|U-\bar{U}\|_2^2},$ where $\bar{U}$ is the spatial average of $U$. The error plots are shown in Figure \ref{Fig:error:curves}. Some of the images of the predicted dynamics are presented in Figure \ref{Fig:Img:Dynamics} and the errors maps are presented in Figure \ref{Fig:Img:Dynamics:errmap}. As we can see that, even trained with noisy data, the PDE-Net 2.0 is able to perform long-term prediction. Having multiple $\delta t$-blocks indeed improves predict accuracy. Furthermore, PDE-Net 2.0 performs significantly better than Frozen-PDE-Net 2.0.
This suggests that when we do not know the PDE, we cannot possibly know how to properly discretize it. Therefore, for inverse problems, it is better to learn both the PDE model and its discretization simultaneously.

\begin{figure}[htp]
\centering
\subfigure{\includegraphics[width = 0.8\textwidth]{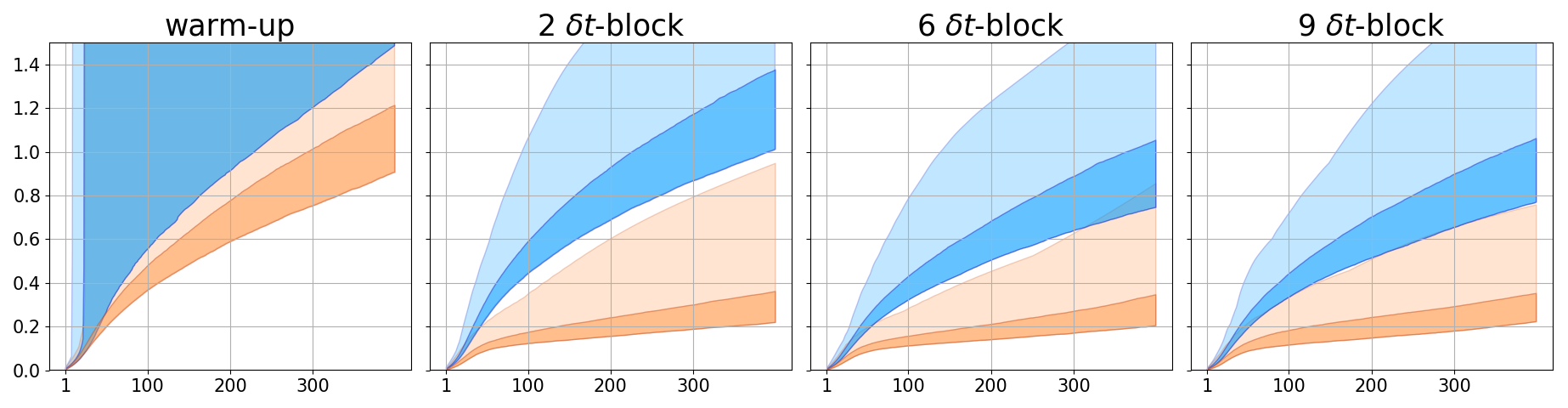}}
\caption{Burgers' equation: Prediction errors of the PDE-Net 2.0(orange) and Frozen-PDE-Net 2.0(blue) with $5\times 5$ filters. In each plot, the horizontal axis indicates the time of prediction in the interval $(0,400\times\delta t]=(0,4]$, and the vertical axis shows the relatively errors. The banded curves indicate the 25\%-100\% percentile of the relative errors among 1000 test samples. The darker regions indicate the 25\%-75\% percentile of the relative errors, which shows that PDE-Net 2.0 can predict very well in most cases.}\label{Fig:error:curves}
\end{figure}

\begin{figure}[htp]
\centering
\subfigure{\includegraphics[width = 0.45\textwidth]{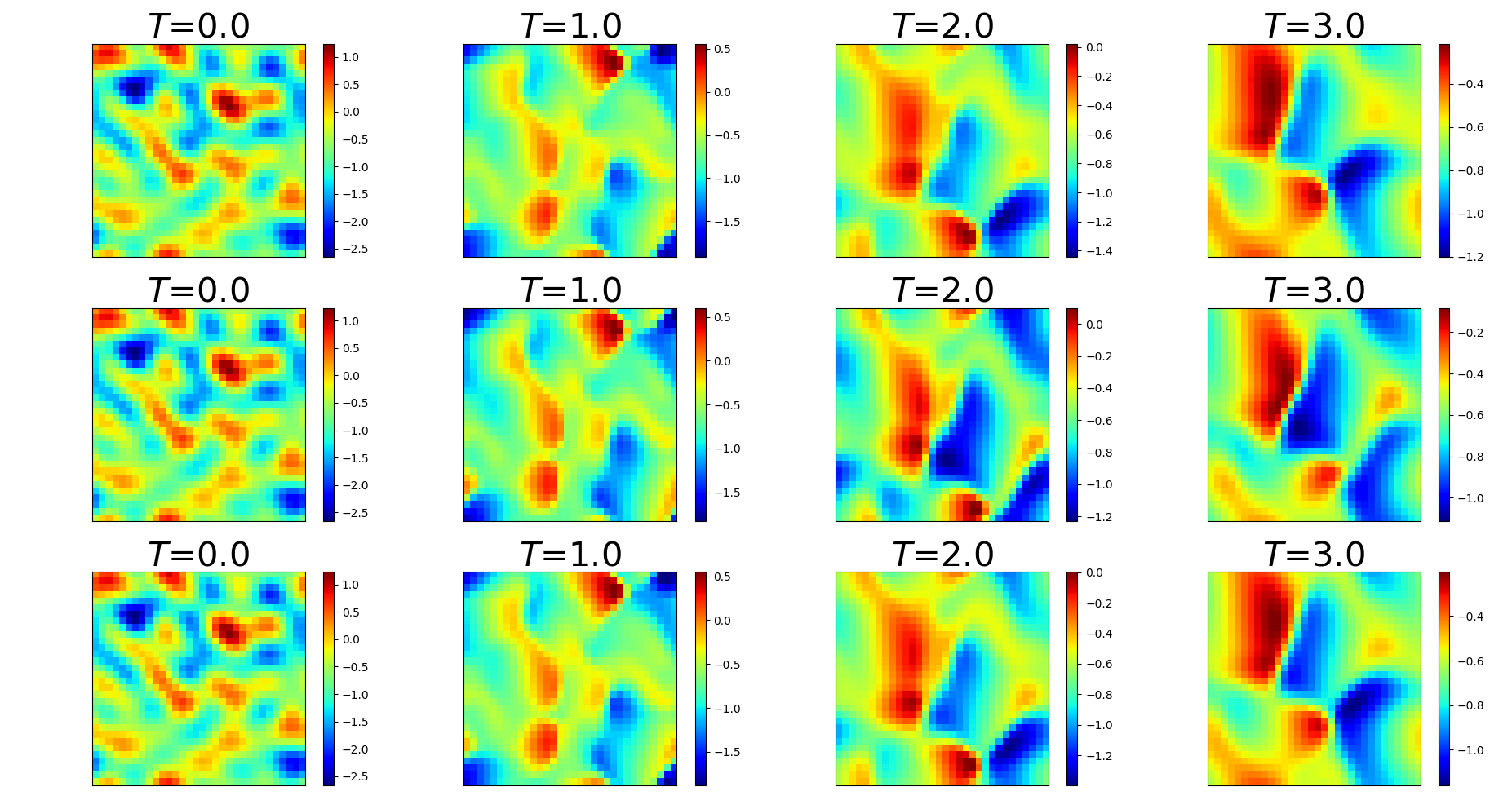}}
\subfigure{\includegraphics[width = 0.45\textwidth]{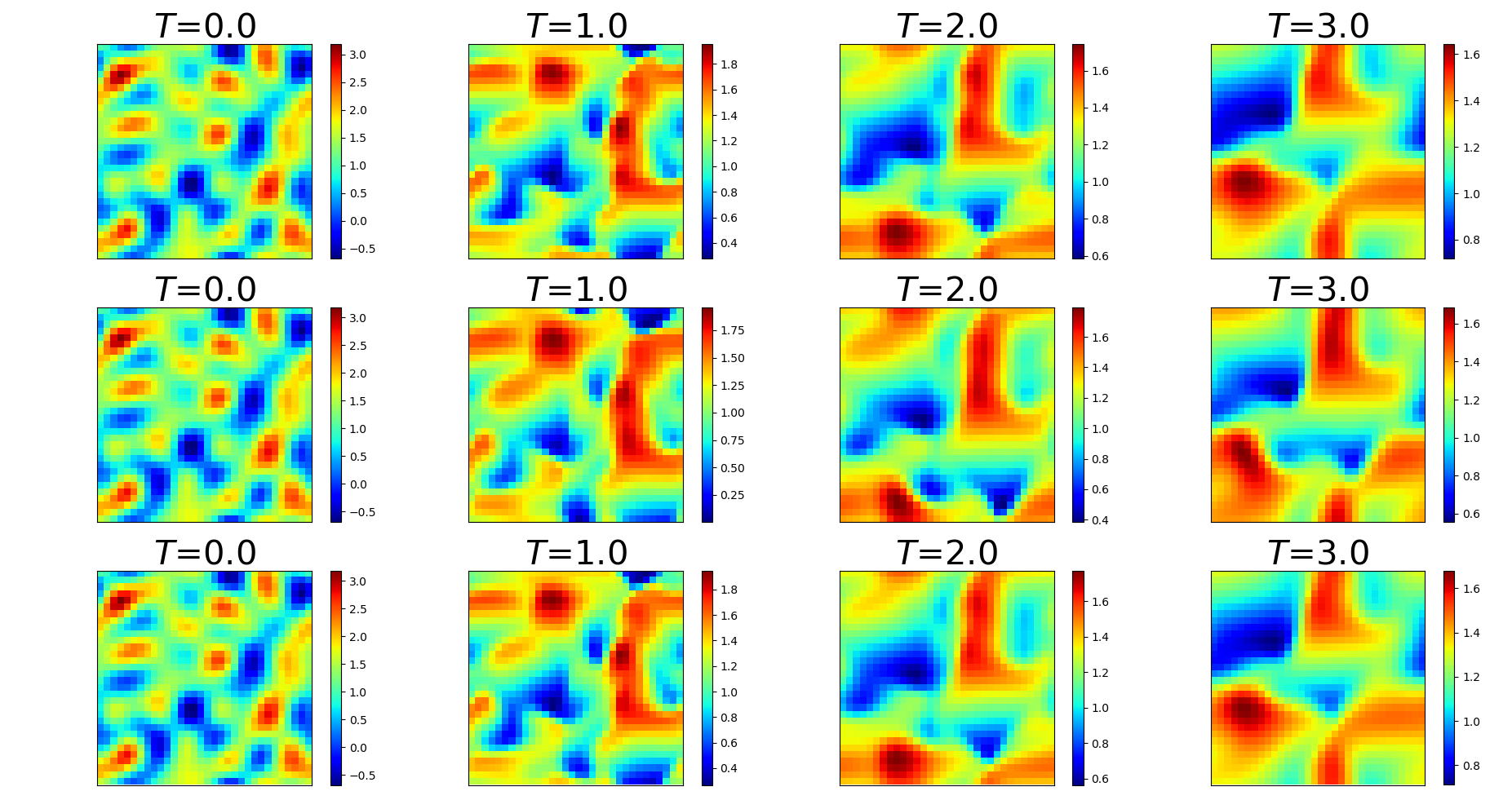}}\\
\subfigure{Dynamics of $u$}\hspace{2in}
\subfigure{Dynamics of $v$}
  \caption{Burgers' equation: The first row shows the images of the true dynamics. The last two rows show the images of the predicted dynamics using the Frozen-PDE-Net 2.0 (the second row) and PDE-Net 2.0 with 9 $\delta t$-blocks (the third row). Time step $\delta t=0.01$.}\label{Fig:Img:Dynamics}
\end{figure}

\begin{figure}[htp]
\centering
\subfigure{\includegraphics[width = 0.45\textwidth]{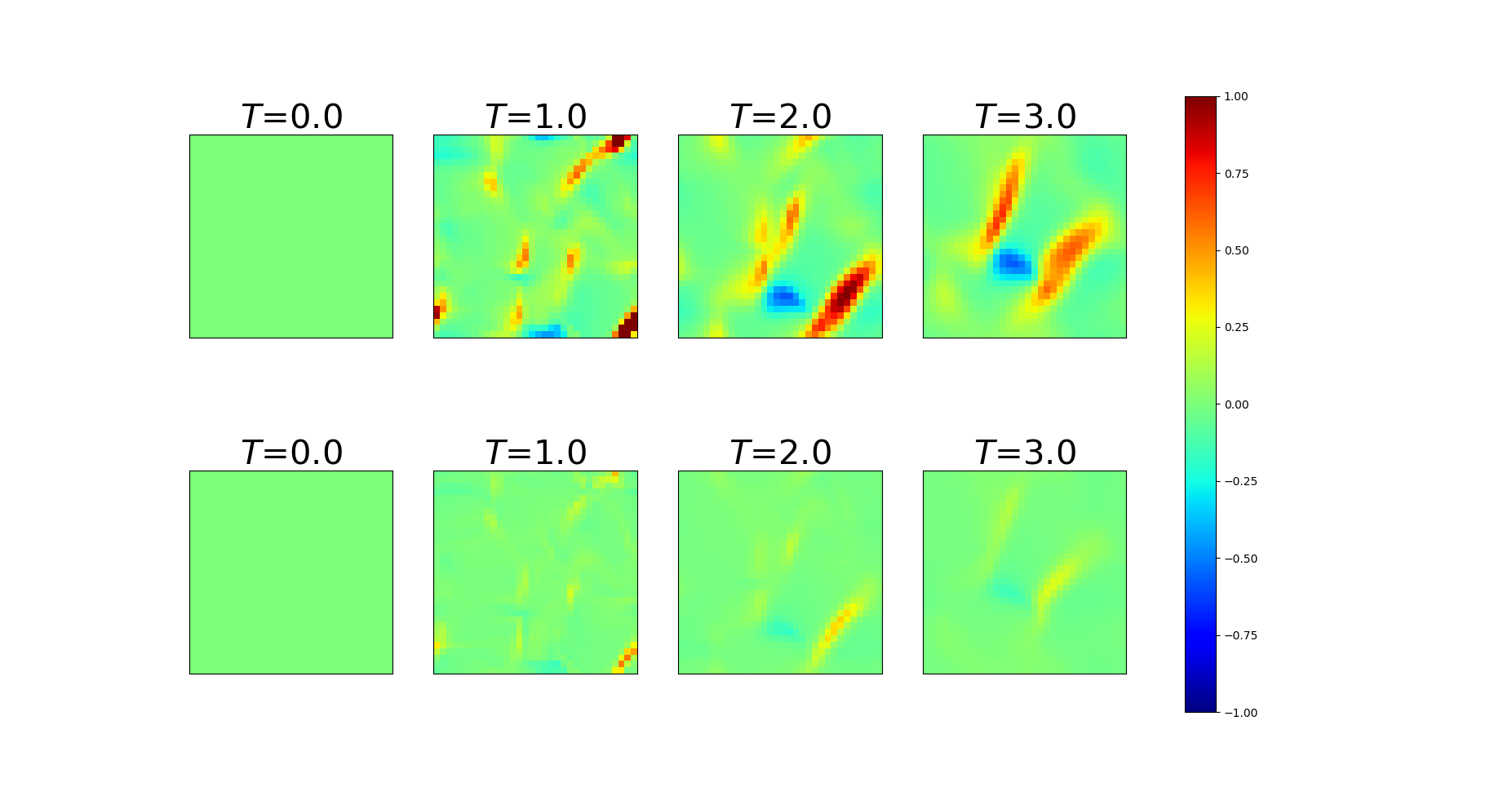}}
\subfigure{\includegraphics[width = 0.45\textwidth]{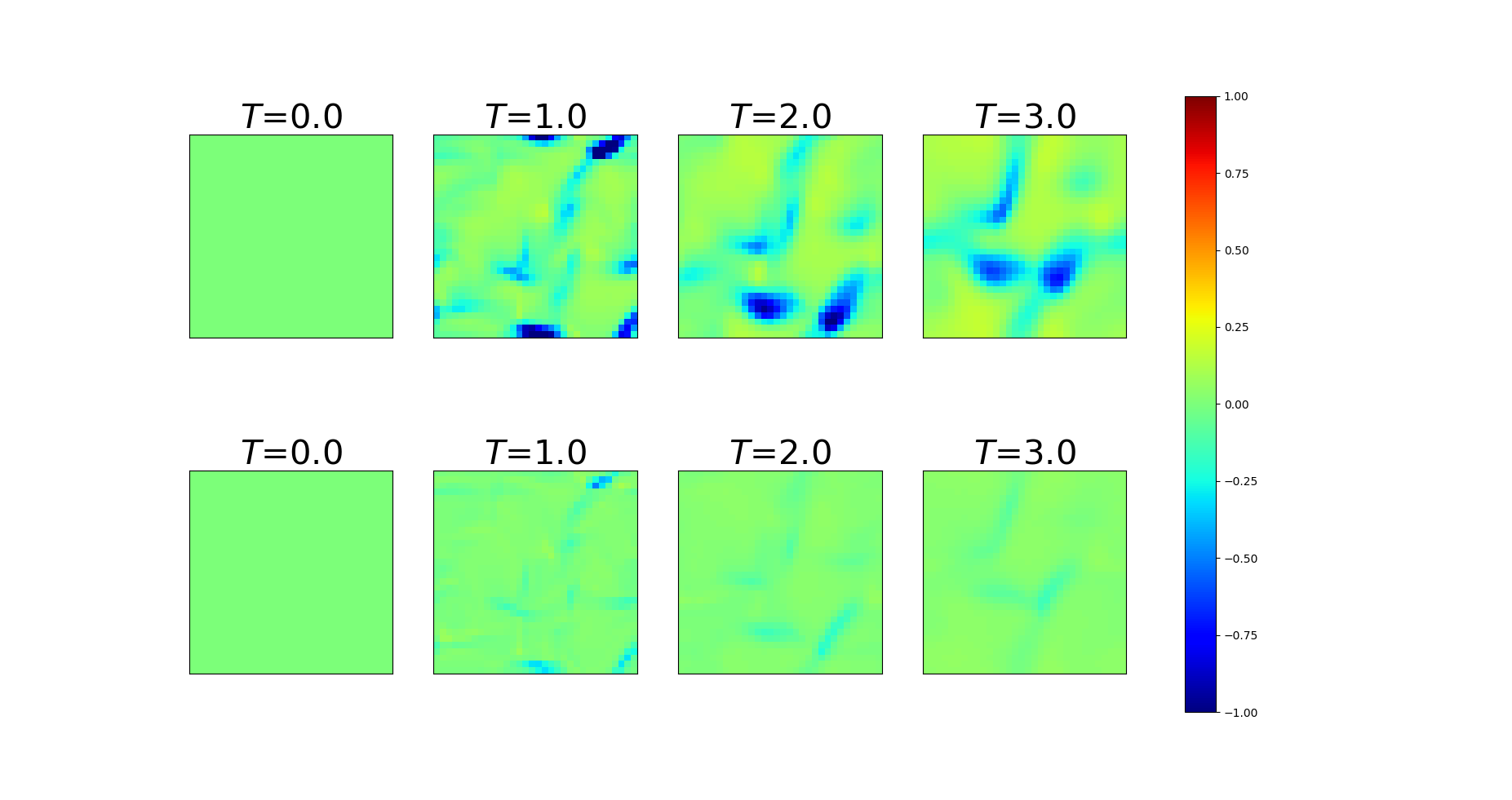}}\\
\subfigure{Error maps of $u$}\hspace{2in}
\subfigure{Error maps of $v$}
  \caption{Burgers' equation: The error maps of the Frozen-PDE-Net 2.0 (the first row) and PDE-Net 2.0 (second row) having 9 $\delta t$-blocks. Time step $\delta t=0.01$.}\label{Fig:Img:Dynamics:errmap}
\end{figure}

\subsection{Importance of \texorpdfstring{$L^{\symnet}$}{L-SymNet} and Pseudo-upwind}

This subsection demonstrates the importance of enforcing sparsity on the $\symnet$ and using pseudo-upwind. As we can see from Figure \ref{Fig:Img:remainder} that having sparsity constraints on the $\symnet$ helps with suppressing the weights on the terms that do not exist in the Burgers' equation. Furthermore, Figure \ref{Fig:error:curves:nosparse} and Figure \ref{Fig:error:curves:central} show that having sparsity constraint on the $\symnet$ or using pseudo-upwind can significantly reduce prediction errors. 

\begin{figure}[htp]
\centering
\subfigure{\includegraphics[width = 0.45\textwidth]{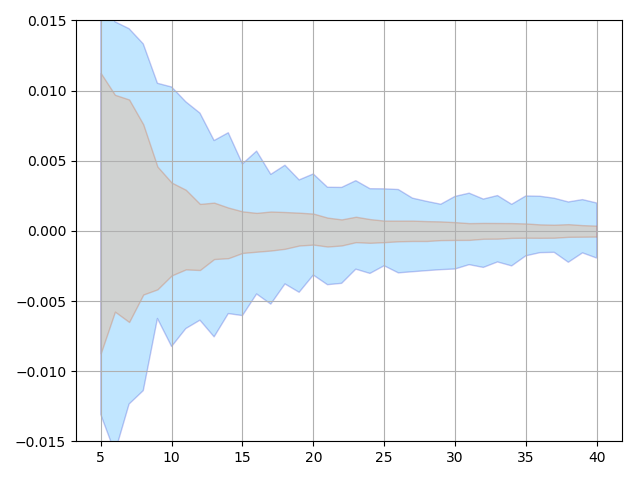}}
\subfigure{\includegraphics[width = 0.45\textwidth]{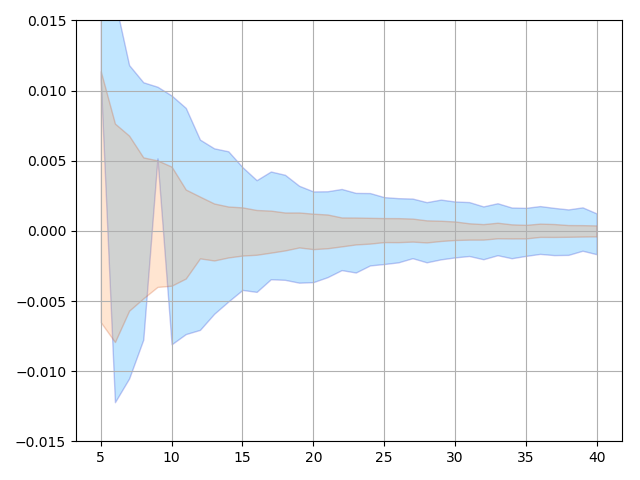}}
  \caption{Burgers' equation: The coefficients of the remainder terms for the $u$-component (left) and $v$-component (right) with the sparsity constraint on the $\symnet$ (orange) and without the sparsity constraint (blue). The bands for both cases are computed based on 20 independent training.}\label{Fig:Img:remainder}
\end{figure}

\begin{figure}[htp]
\centering
\subfigure{\includegraphics[width = 0.8\textwidth]{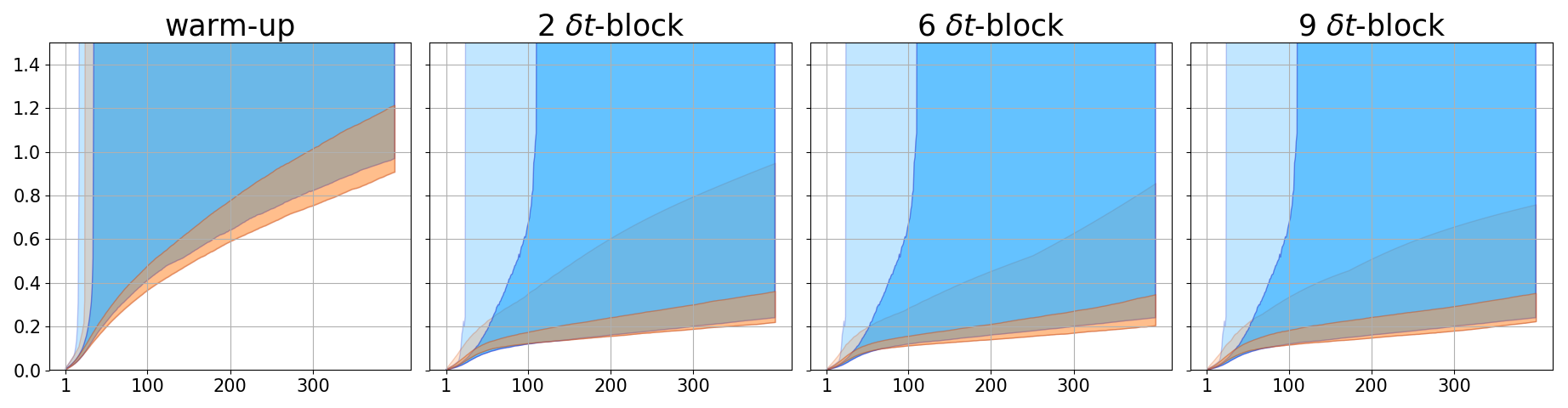}}
\caption{Burgers' equation: Prediction errors of the PDE-Net 2.0 with (orange) and without (blue) sparsity constraint on the $\symnet$. In each plot, the horizontal axis indicates the time of prediction in the interval $(0,400\times\delta t]=(0,4]$, and the vertical axis shows the relatively errors. The banded curves indicate the 25\%-100\% percentile of the relative errors among 1000 test samples.}\label{Fig:error:curves:nosparse}
\end{figure}

\begin{figure}[htp]
\centering
\subfigure{\includegraphics[width = 0.8\textwidth]{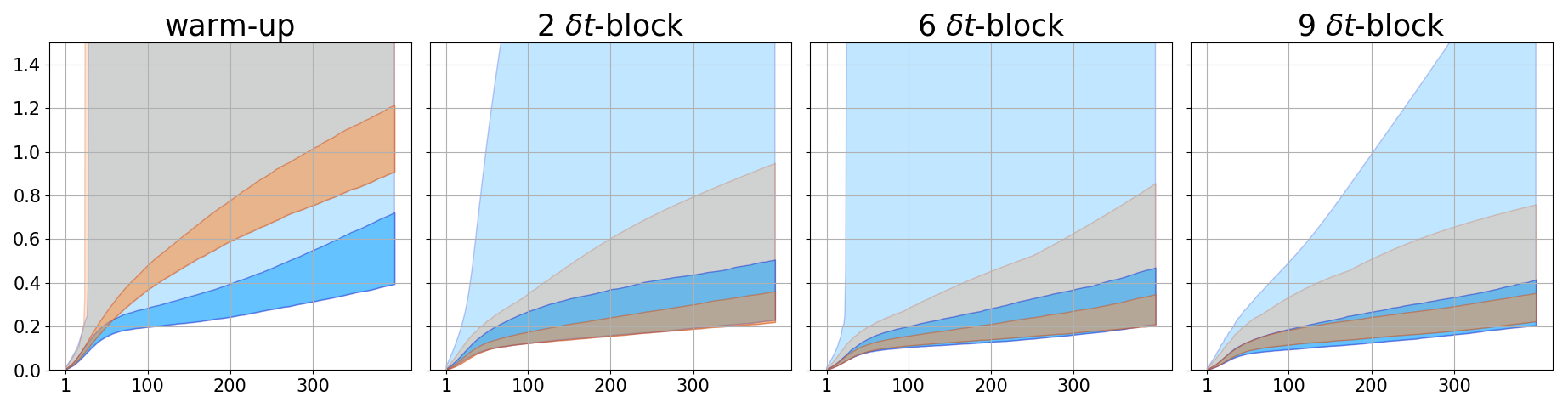}}
\caption{Burgers' equation: Prediction errors of the PDE-Net 2.0 with (orange) and without (blue) pseudo-upwind in each $\delta t$-block. In each plot, the horizontal axis indicates the time of prediction in the interval $(0,400\times\delta t]=(0,4]$, and the vertical axis shows the relatively errors. The banded curves indicate the 25\%-100\% percentile of the relative errors among 1000 test samples.}\label{Fig:error:curves:central}
\end{figure}

\section{Numerical Studies: Diffusion Equation}\label{Sec:Results:Heat}
Diffusion phenomenon has been studied in many applications in physics e.g. the collective motion of micro-particles in materials due to random movement of each particle, or modeling the distribution of temperature in a given region over time.

\subsection{Simulated data, training and testing}
Consider the 2-dimensional heat equation with periodic boundary condition on $\Omega=[0,2\pi]\times [0,2\pi]$
\begin{equation}\label{E:testPDE:Heat}
\left\{
\begin{array}{ll}
\frac{\partial u}{\partial t}&=c\Delta u,(t,x,y)\in [0, 1]\times\Omega,\\
u|_{t=0}&=u_0(x,y),
\end{array}
\right.
\end{equation}
where $c=0.1$. The training data of the heat equation is generated by 2nd order Runge-Kutta in time with $\delta t = \frac{1}{1600}$, and central difference scheme in space on a $128 \times 128$ mesh. We then restrict the data to a $32 \times 32$ mesh. The initial value $u_0(x,y)$ is also generated from \eqref{E:random:initialize}.

\subsection{Results and discussions}

The demonstration on the ability of the trained PDE-Net 2.0 to identify the PDE model is given in Table \ref{Tab:Analytic:Heat}. As one can see from Table \ref{Tab:Analytic:Heat} that we can recover the terms of the heat equation with good accuracy. Furthermore, all the terms that are not included in the heat equation have much smaller weights in the $\symnet$.

We also demonstrate the ability of the trained PDE-Net 2.0 in prediction. The testing method is exactly the same as the method described in Section \ref{Sec:Results:Burgers}. Comparisons between PDE-Net 2.0 and Frozen-PDE-Net 2.0 are shown in Figure \ref{Fig:error:curves:Heat}, where we can clearly see the advantage of learning the filters. Visualization of the predicted dynamics is given in Figure \ref{Fig:Img:Dynamics:Heat}. All these results show that the learned PDE-Net 2.0 performs well in prediction.

\begin{table}[htp]
\centering
\caption{PDE model identification. Note that the largest term of the remainders (i.e. the ones that are not included in the table) is 8e-4$u_y$ (for Frozen-PDE-Net 2.0) and 6e-5$u$ (for PDE-Net 2.0).}\label{Tab:Analytic:Heat}
\begin{tabular}{|c|l|}
\hline
Correct PDE & $u_t=0.1(u_{xx}+u_{yy})$ \\
\hline
Frozen-PDE-Net 2.0 & $u_t=0.103u_{xx}+0.103u_{yy}$ \\
\hline
PDE-Net 2.0 & $u_t=0.999u_{xx}+0.998u_{yy}$ \\
\hline
\end{tabular}
\end{table}

\begin{figure}[htp]
\centering
\includegraphics[width = 0.8\textwidth]{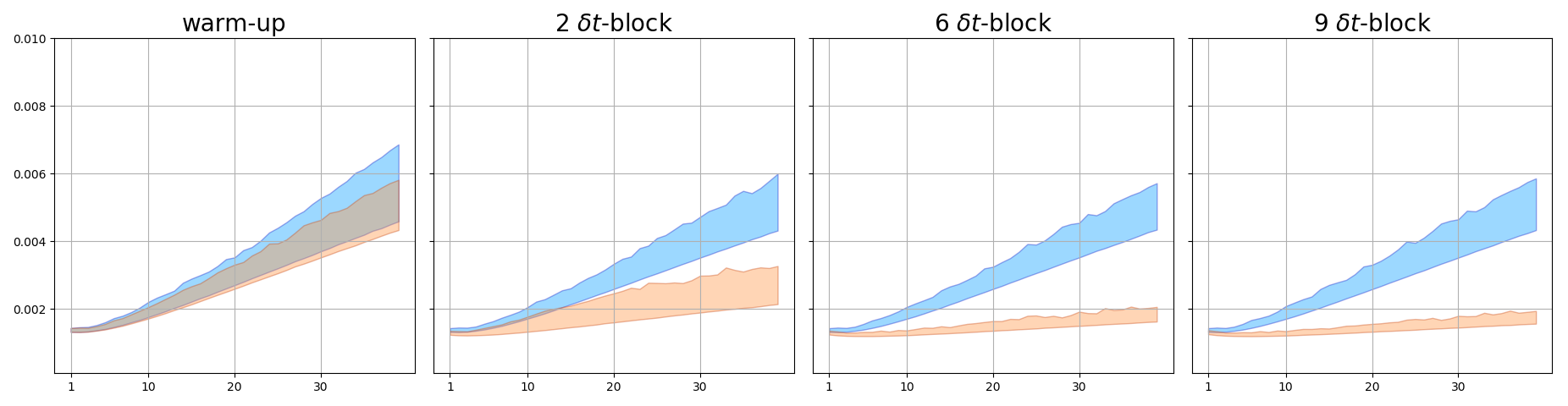}
\caption{Diffusion equation: Prediction errors of the PDE-Net 2.0 (orange) and Frozen-PDE-Net 2.0 (blue). In each plot, the horizontal axis indicates the time of prediction in the interval $(0,150\times\delta t]=(0,1.5]$, and the vertical axis shows the relative errors. The banded curves indicate the 25\%-100\% percentile of the relative errors among 1000 test samples.}\label{Fig:error:curves:Heat}
\end{figure}

\begin{figure}[htp]
\centering
\includegraphics[width = 0.45\textwidth]{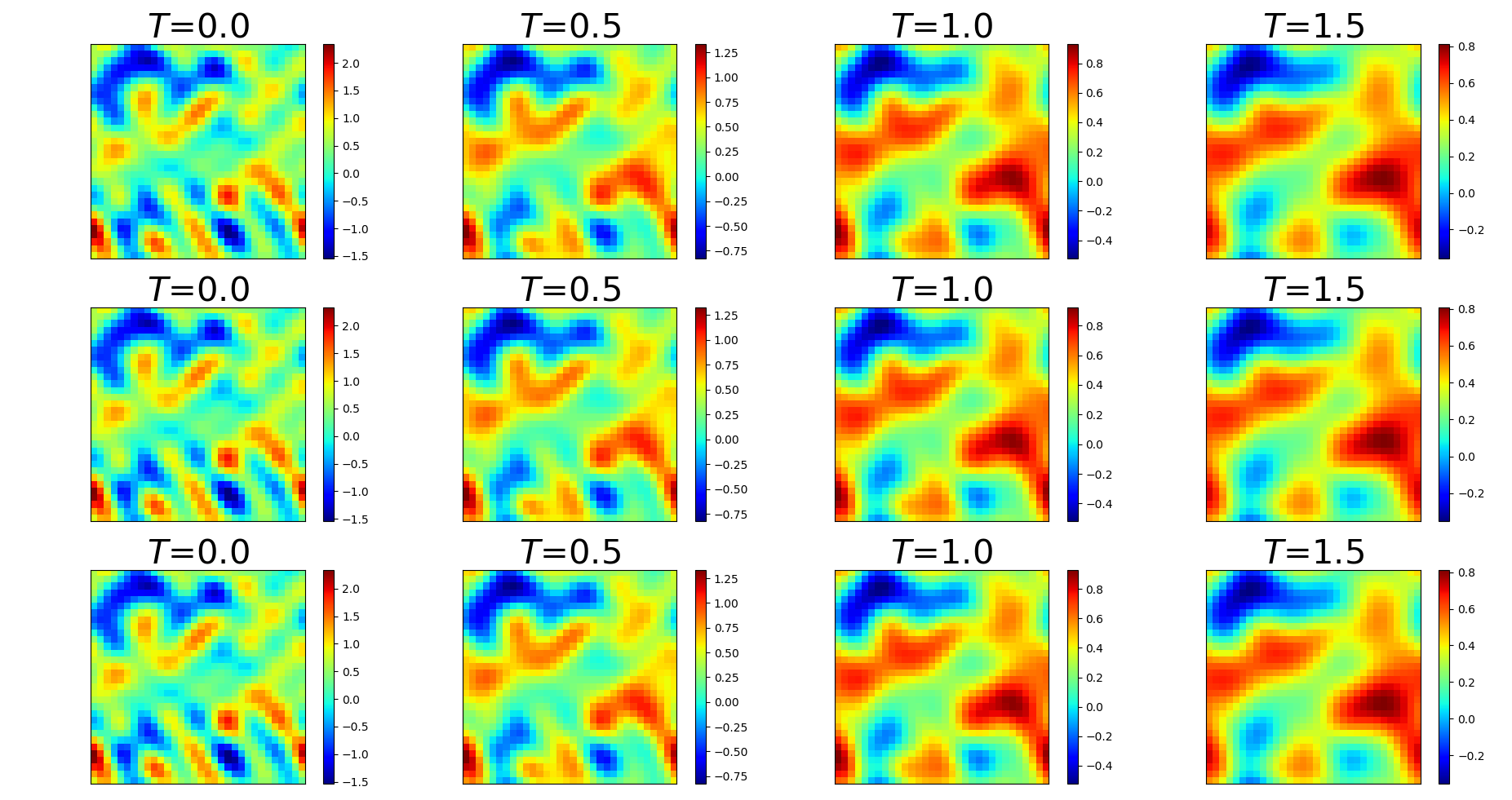}
\caption{Diffusion equation: The first row shows the images of the true dynamics. The second row shows the predicted dynamics using the Frozen-PDE-Net 2.0 with 9 $\delta t$-blocks. The third row shows the predicted dynamics using the PDE-Net 2.0 with 9 $\delta t$-blocks. Here, $\delta t=0.01$.}\label{Fig:Img:Dynamics:Heat}
\end{figure}

\section{Numerical Studies: Convection Diffusion Equation with A Reactive Source}\label{Sec:Results:CDR}
Convection diffusion systems are mathematical models which correspond to the transferring of some physical quantities such as energy or materials due to diffusion and convection. Specifically, a convection diffusion system with a reactive source can be used to model a large range of chemical systems in which the transferring of materials competes with productions of materials induced by several chemical reactions.

\subsection{Simulated data, training and testing}
Consider a 2-dimensional convection diffusion equation with a reactive source and the periodic boundary condition on $\Omega=[0,2\pi]\times[0,2\pi]$:
\begin{equation}\label{E:testPDE:CDR}
  \begin{split}
    u_t&=-uu_x-vu_y+\nu \Delta u+\lambda(A)u-\omega(A)v, \\
    v_t&=-uv_x-vv_y+\nu \Delta v+\omega(A)u+\lambda(A)v, \\
    & A^2=u^2+v^2,\omega=-\beta A^2, \lambda=1-A^2,
  \end{split}
\end{equation}
where $(t,x,y)\in [0, 1.5]\times\Omega,$ and $\nu=0.1,\beta=1$.
Training data is generated the same way as what we did for Burgers' equation in Section \ref{Sec:Results:Burgers}. A 2nd order Runge-Kutta with time step $\delta t=1/10000$ is adopted for temporal discretization. We choose a 2nd order upwind scheme for the convection terms and the central difference scheme for $\Delta$ on a $128 \times 128$ mesh. We then restrict the data to a $32 \times 32$ mesh. Noise is added the same way as the Burgers' equation. The initial values $u_0(x,y), v_0(x,y)$ are also generated from \eqref{E:random:initialize}.

\subsection{Results and discussions}

The capability of the trained PDE-Net 2.0 to identify the underlying PDE model is demonstrated in Table \ref{Tab:Analytic:CDR}. As one can see that we can recover the terms of the reaction convection diffusion equation with good accuracy. Furthermore, all the terms that are not included in this equation have relatively small weights in the $\symnet$.

We also demonstrate the ability of the trained PDE-Net 2.0 in prediction. The testing method is exactly the same as the method described in Section \ref{Sec:Results:Burgers}. Comparisons between PDE-Net 2.0 and Frozen-PDE-Net 2.0 are shown in Figure \ref{Fig:error:curves:CDR}. Visualization of the predicted dynamics and errors maps are given in Figure \ref{Fig:Img:Dynamics:CDR} and Figure \ref{Fig:Img:Dynamics:errmap:CDR}. Similar to what we observed in Section \ref{Sec:Results:Burgers}, we can clearly see the benefit from learning discretizations. PDE-Net 2.0 obtains more accurate estimations of the coefficients for the nonlinear convection terms (i.e. the term $-uu_x-vu_y$ in Table \ref{Tab:Analytic:CDR}) and makes more accurate predictions (Figure \ref{Fig:error:curves:CDR}) than  Frozen-PDE-Net 2.0.

\begin{table}[htp]
\centering
  \caption{PDE model identification.}\label{Tab:Analytic:CDR}
\begin{tabular}{|c|l|}
\hline
  \multirow{2}{*}{Correct PDE} & $u_t=-uu_x-vu_y+0.1\Delta u+(v-u)(u^2+v^2)+u$ \\
  & $v_t=-uv_x-vv_y+0.1\Delta v-(v+u)(u^2+v^2)+v$ \\
\hline
  \multirow{4}{*}{Frozen-PDE-Net 2.0} & $u_t=-{\color{blue}0.86}uu_x-{\color{blue}0.90}vu_y+0.09u_{xx}+0.09u_{yy}$ \\
                                      & $\ \ \ \ \ +1.02u^2v-1.02u^3-1.01uv^2+1.01u+0.99v^3$ \\
                                      & $v_t=-{\color{blue}0.87}uv_x-{\color{blue}0.85}vv_y+0.09v_{xx}+0.09v_{yy}$ \\
                                      & $\ \ \ \ \ +1.04u^2v-1.02uv^2-1.01v^3+0.99v-0.99u^3$ \\
\hline
  \multirow{4}{*}{PDE-Net 2.0} & $u_t=-{\color{orange}0.98}vu_y-{\color{orange}0.93}uu_x+0.10u_{xx}+0.10u_{yy}$ \\
                               & $\ \ \ \ \ -1.05uv^2+0.99v^3-0.98u^3+0.98u+0.97u^2v$ \\ 
                               & $v_t=-{\color{orange}0.99}uv_x-{\color{orange}0.96}vv_y+0.10v_{yy}+0.10v_{xx}$ \\
                               & $\ \ \ \ \ -1.04u^2v-1.02v^2-1.02uv^2+1.01v-1.00u^3$ \\
\hline
\end{tabular}
\end{table}

\begin{figure}[htp]
\centering
\includegraphics[width = 0.8\textwidth]{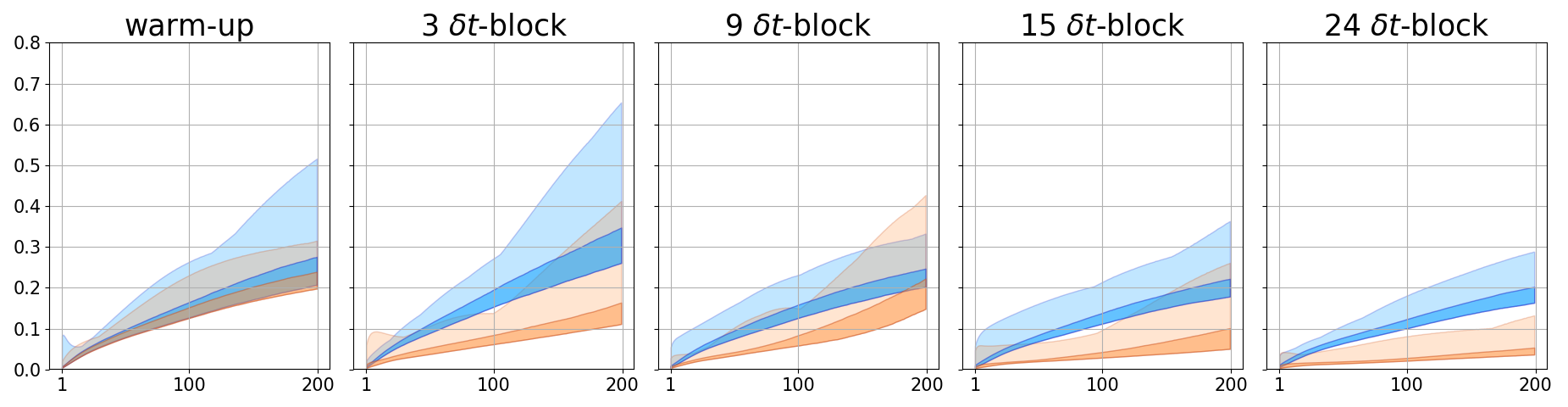}
\caption{Reaction convection diffusion: Prediction errors of the PDE-Net 2.0 (orange) and Frozen-PDE-Net 2.0 (blue). In each plot, the horizontal axis indicates the time of prediction in the interval $(0,200\times\delta t]=(0,2]$, and the vertical axis shows the relative errors. The banded curves indicate the 25\%-100\% percentile of the relative errors among 1000 test samples.}\label{Fig:error:curves:CDR}
\end{figure}

\begin{figure}[htp]
\centering
\includegraphics[width = 0.45\textwidth]{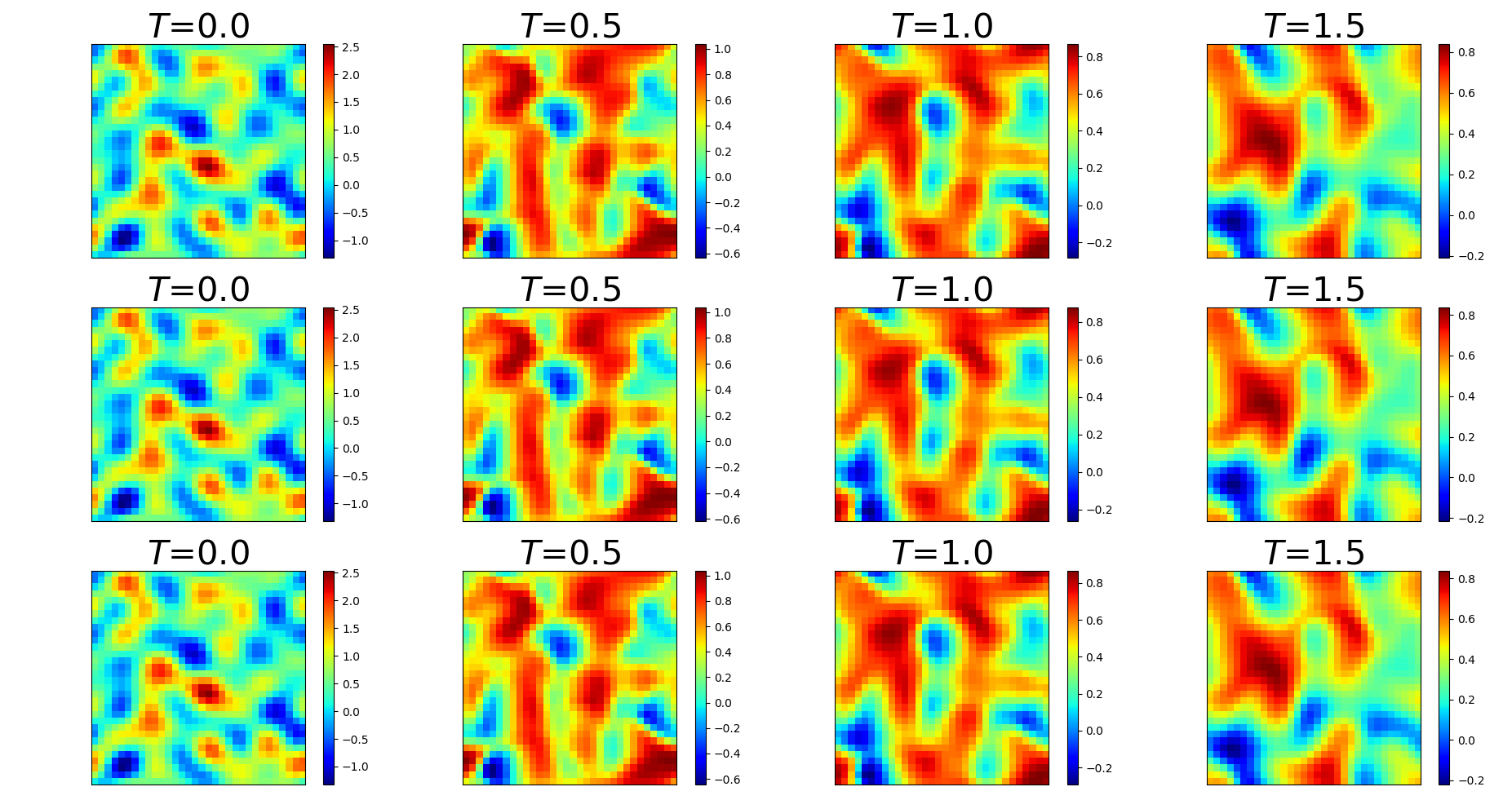}
\includegraphics[width = 0.45\textwidth]{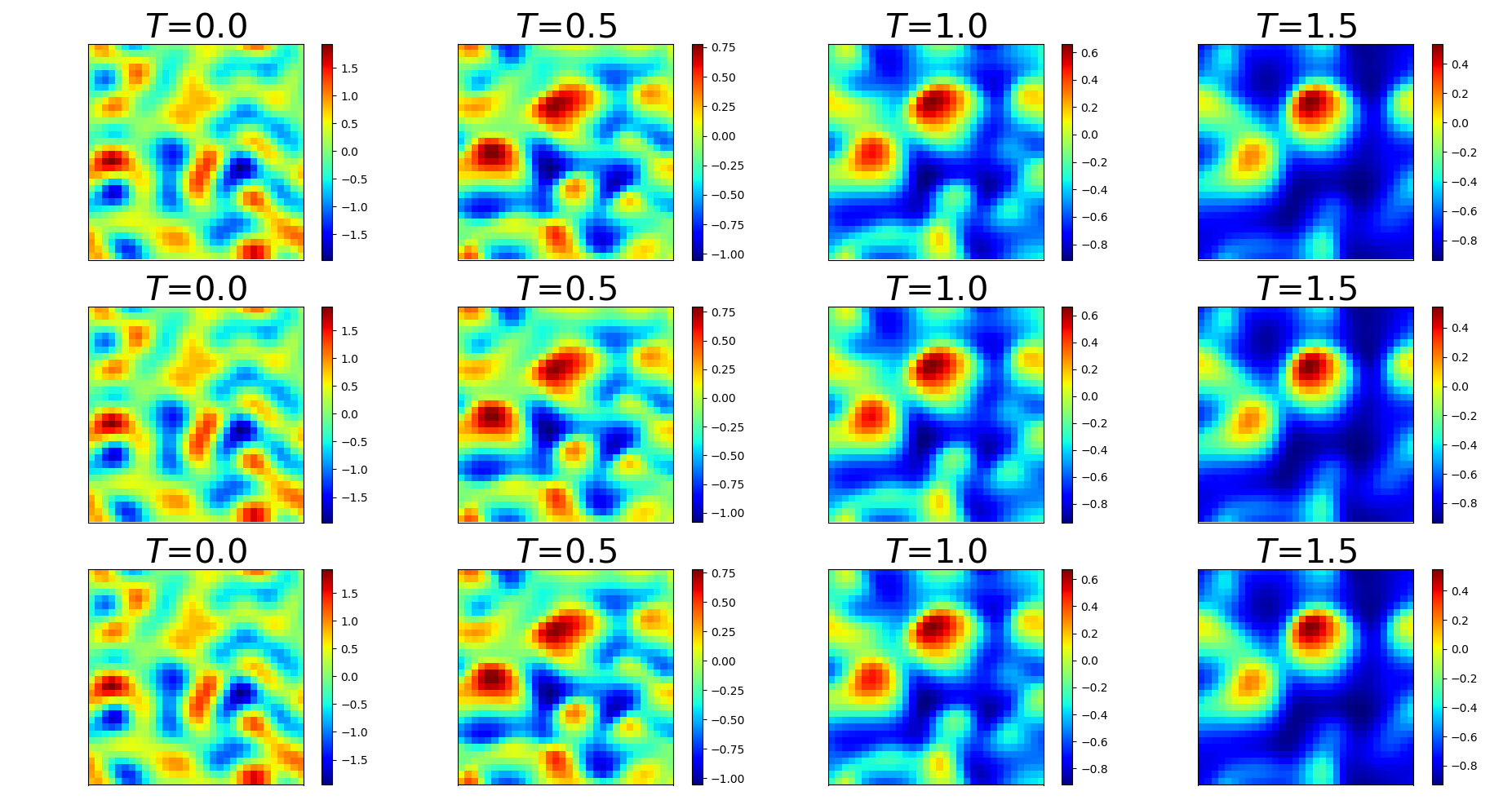}\\
\subfigure{Dynamics of $u$}\hspace{2in}
\subfigure{Dynamics of $v$}
\caption{Reaction convection diffusion: The first row shows the images of the true dynamics. The second row shows the predicted dynamics using the Frozen-PDE-Net 2.0 with 24 $\delta t$-blocks. The third row shows the predicted dynamics using the PDE-Net 2.0 with 24 $\delta t$-blocks. Here, $\delta t=0.01$.}\label{Fig:Img:Dynamics:CDR}
\end{figure}

\begin{figure}[htp]
\centering
\subfigure{\includegraphics[width = 0.45\textwidth]{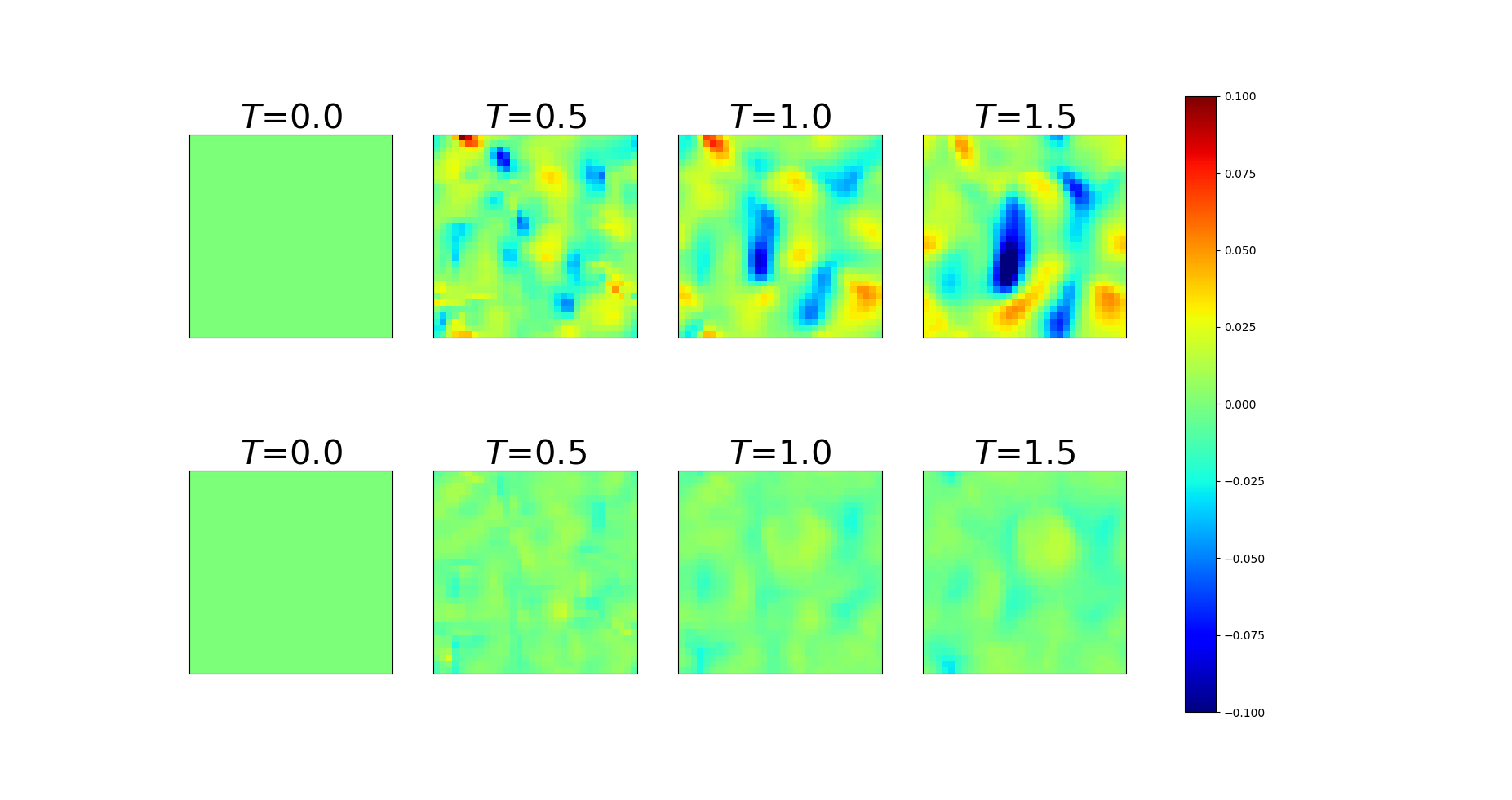}}
\subfigure{\includegraphics[width = 0.45\textwidth]{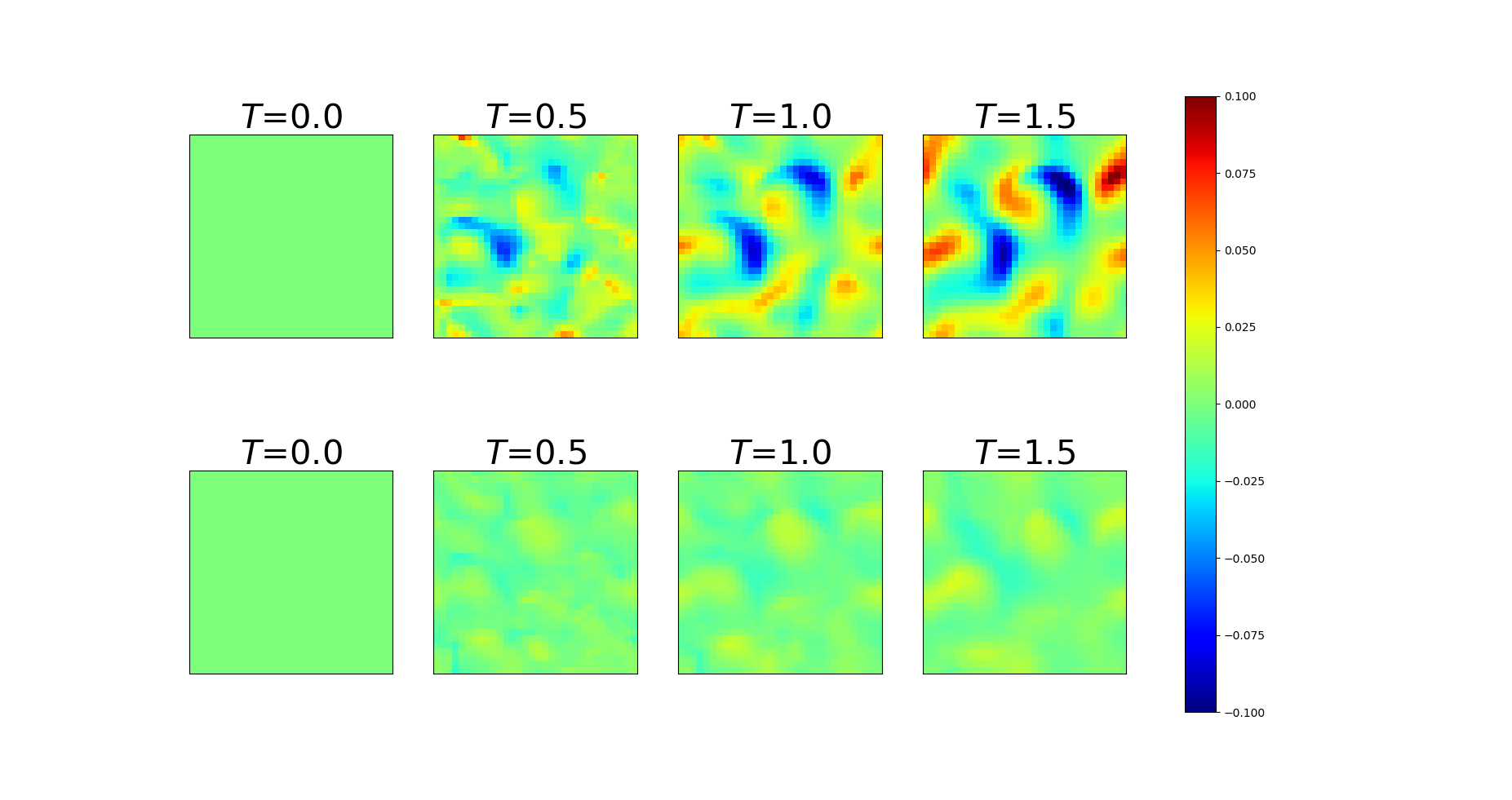}}\\
\subfigure{Error maps of $u$}\hspace{2in}
\subfigure{Error maps of $v$}
  \caption{Reaction convection diffusion: The error maps of the Frozen-PDE-Net 2.0 (the first row) and PDE-Net 2.0 (second row) having 24 $\delta t$-blocks. Time step $\delta t=0.01$.}\label{Fig:Img:Dynamics:errmap:CDR}
\end{figure}

\section{Conclusions and Future Work}

In this paper, we proposed a numeric-symbolic hybrid deep network, called PDE-Net 2.0, for PDE model recovery from observed dynamic data. PDE-Net 2.0 is able to recover the analytic form of the PDE model with minor assumptions on the mechanisms of the observed dynamics. For example, it is able to recover the analytic form of Burgers' equation with good confidence without any prior knowledge on the type of the equation. Therefore, PDE-Net 2.0 has the potential to uncover potentially new PDEs from observed data. Furthermore, after training, the network can perform accurate long-term prediction without re-training for new initial conditions. The limitations and possible future extensions of the current version of PDE-Net 2.0 is twofold: 1) having only addition and multiplication in the $\symnet$ may still be too restrictive, and one may include division as an additional operation in $\symnet$ to further improve its expressive power; 2) using forward Euler as temporal discretization is the most straightforward treatment, while a more sophisticated temporal scheme may help with the model recovery and prediction. Both of these worth further exploration. Furthermore, we would like to apply the network to real biological dynamic data. We will further investigate the reliability of the network and explore the possibility to uncover new dynamical principles that have meaningful scientific explanations.

\section*{Acknowledgments}
Zichao Long is supported by The Elite Program of Computational and Applied Mathematics for PhD Candidates of Peking University. Yiping Lu is supported by the Elite Undergraduate Training Program of the School of Mathematical Sciences at Peking University. Bin Dong is supported in part by Beijing Natural Science Foundation (No. 180001) and Beijing Academy of Artificial Intelligence (BAAI).

\bibliographystyle{unsrt}  

\bibliography{ref}

\begin{thebibliography}{10}

\bibitem{Long2018PDE}
Zichao Long, Yiping Lu, Xianzhong Ma, and Bin Dong.
\newblock Pde-net: Learning pdes from data.
\newblock In {\em International Conference on Machine Learning}, pages
  3214--3222, 2018.

\bibitem{brennen2005fundamentals}
Christopher~Earls Brennen and Christopher~E Brennen.
\newblock {\em Fundamentals of multiphase flow}.
\newblock Cambridge university press, 2005.

\bibitem{efendievmathematical}
Messoud Efendiev.
\newblock Mathematical modeling of mitochondrial swelling.

\bibitem{bongard2007automated}
Josh Bongard and Hod Lipson.
\newblock Automated reverse engineering of nonlinear dynamical systems.
\newblock {\em Proceedings of the National Academy of Sciences},
  104(24):9943--9948, 2007.

\bibitem{schmidt2009distilling}
Michael Schmidt and Hod Lipson.
\newblock Distilling free-form natural laws from experimental data.
\newblock {\em science}, 324(5923):81--85, 2009.

\bibitem{raissi2018hidden}
Maziar Raissi and George~Em Karniadakis.
\newblock Hidden physics models: Machine learning of nonlinear partial
  differential equations.
\newblock {\em Journal of Computational Physics}, 357:125--141, 2018.

\bibitem{raissi2017physicsII}
Maziar Raissi, Paris Perdikaris, and George~Em Karniadakis.
\newblock Physics informed deep learning (part ii): data-driven discovery of
  nonlinear partial differential equations.
\newblock {\em arXiv preprint arXiv:1711.10566}, 2017.

\bibitem{brunton2016discovering}
Steven~L Brunton, Joshua~L Proctor, and J~Nathan Kutz.
\newblock Discovering governing equations from data by sparse identification of
  nonlinear dynamical systems.
\newblock {\em Proceedings of the National Academy of Sciences}, page
  201517384, 2016.

\bibitem{schaeffer2017learning}
Hayden Schaeffer.
\newblock Learning partial differential equations via data discovery and sparse
  optimization.
\newblock {\em Proc. R. Soc. A}, 473(2197):20160446, 2017.

\bibitem{rudy2017data}
Samuel~H Rudy, Steven~L Brunton, Joshua~L Proctor, and J~Nathan Kutz.
\newblock Data-driven discovery of partial differential equations.
\newblock {\em Science Advances}, 3(4):e1602614, 2017.

\bibitem{chang2018identification}
Haibin Chang and Dongxiao Zhang.
\newblock Identification of physical processes via combined data-driven and
  data-assimilation methods.
\newblock {\em arXiv preprint arXiv:1810.11977}, 2018.

\bibitem{schaeffer2018extracting}
Hayden Schaeffer, Giang Tran, Rachel Ward, and Linan Zhang.
\newblock Extracting structured dynamical systems using sparse optimization
  with very few samples.
\newblock {\em arXiv preprint arXiv:1805.04158}, 2018.

\bibitem{wu2019learning}
Zongmin Wu and Ran Zhang.
\newblock Learning physics by data for the motion of a sphere falling in a
  non-newtonian fluid.
\newblock {\em Communications in Nonlinear Science and Numerical Simulation},
  67:577--593, 2019.

\bibitem{de2017deep}
Emmanuel de~Bezenac, Arthur Pajot, and Patrick Gallinari.
\newblock Deep learning for physical processes: Incorporating prior scientific
  knowledge.
\newblock {\em arXiv preprint arXiv:1711.07970}, 2017.

\bibitem{he2016deep}
Kaiming He, Xiangyu Zhang, Shaoqing Ren, and Jian Sun.
\newblock Deep residual learning for image recognition.
\newblock In {\em Proceedings of the IEEE conference on computer vision and
  pattern recognition}, pages 770--778, 2016.

\bibitem{he2016identity}
Kaiming He, Xiangyu Zhang, Shaoqing Ren, and Jian Sun.
\newblock Identity mappings in deep residual networks.
\newblock In {\em European conference on computer vision}, pages 630--645.
  Springer, 2016.

\bibitem{chen2015learning}
Yunjin Chen, Wei Yu, and Thomas Pock.
\newblock On learning optimized reaction diffusion processes for effective
  image restoration.
\newblock In {\em Proceedings of the IEEE conference on computer vision and
  pattern recognition}, pages 5261--5269, 2015.

\bibitem{weinan2017proposal}
E~Weinan.
\newblock A proposal on machine learning via dynamical systems.
\newblock {\em Communications in Mathematics and Statistics}, 5(1):1--11, 2017.

\bibitem{haber2017stable}
Eldad Haber and Lars Ruthotto.
\newblock Stable architectures for deep neural networks.
\newblock {\em Inverse Problems}, 34(1):014004, 2017.

\bibitem{sonoda2017double}
Sho Sonoda and Noboru Murata.
\newblock Double continuum limit of deep neural networks.
\newblock In {\em ICML Workshop Principled Approaches to Deep Learning}, 2017.

\bibitem{Lu2018Beyond}
Yiping Lu, Aoxiao Zhong, Quanzheng Li, and Bin Dong.
\newblock Beyond finite layer neural networks: Bridging deep architectures and
  numerical differential equations.
\newblock In {\em International Conference on Machine Learning}, pages
  3282--3291, 2018.

\bibitem{chang2017multi}
Bo~Chang, Lili Meng, Eldad Haber, Frederick Tung, and David Begert.
\newblock Multi-level residual networks from dynamical systems view.
\newblock In {\em ICLR}, 2018.

\bibitem{chen2018neural}
Tian~Qi Chen, Yulia Rubanova, Jesse Bettencourt, and David Duvenaud.
\newblock Neural ordinary differential equations.
\newblock {\em arXiv preprint arXiv:1806.07366}, 2018.

\bibitem{qin2018data}
Tong Qin, Kailiang Wu, and Dongbin Xiu.
\newblock Data driven governing equations approximation using deep neural
  networks.
\newblock {\em arXiv preprint arXiv:1811.05537}, 2018.

\bibitem{wiewel2018latent}
Steffen Wiewel, Moritz Becher, and Nils Thuerey.
\newblock Latent-space physics: Towards learning the temporal evolution of
  fluid flow.
\newblock {\em arXiv preprint arXiv:1802.10123}, 2018.

\bibitem{kim2018deep}
Byungsoo Kim, Vinicius~C Azevedo, Nils Thuerey, Theodore Kim, Markus Gross, and
  Barbara Solenthaler.
\newblock Deep fluids: A generative network for parameterized fluid
  simulations.
\newblock {\em arXiv preprint arXiv:1806.02071}, 2018.

\bibitem{duraisamy2015new}
Karthikeyan Duraisamy, Ze~J Zhang, and Anand~Pratap Singh.
\newblock New approaches in turbulence and transition modeling using
  data-driven techniques.
\newblock In {\em 53rd AIAA Aerospace Sciences Meeting}, page 1284, 2015.

\bibitem{duraisamy2019turbulence}
Karthik Duraisamy, Gianluca Iaccarino, and Heng Xiao.
\newblock Turbulence modeling in the age of data.
\newblock {\em Annual Review of Fluid Mechanics}, 51:357--377, 2019.

\bibitem{ma2018model}
Chao Ma, Jianchun Wang, et~al.
\newblock Model reduction with memory and the machine learning of dynamical
  systems.
\newblock {\em arXiv preprint arXiv:1808.04258}, 2018.

\bibitem{noid2008multiscale}
WG~Noid, Jhih-Wei Chu, Gary~S Ayton, Vinod Krishna, Sergei Izvekov, Gregory~A
  Voth, Avisek Das, and Hans~C Andersen.
\newblock The multiscale coarse-graining method. i. a rigorous bridge between
  atomistic and coarse-grained models.
\newblock {\em The Journal of chemical physics}, 128(24):244114, 2008.

\bibitem{zhang2018deep}
Linfeng Zhang, Jiequn Han, Han Wang, Roberto Car, and E~Weinan.
\newblock Deep potential molecular dynamics: a scalable model with the accuracy
  of quantum mechanics.
\newblock {\em Physical review letters}, 120(14):143001, 2018.

\bibitem{cai2012image}
Jian-Feng Cai, Bin Dong, Stanley Osher, and Zuowei Shen.
\newblock Image restoration: total variation, wavelet frames, and beyond.
\newblock {\em Journal of the American Mathematical Society}, 25(4):1033--1089,
  2012.

\bibitem{dong2017image}
Bin Dong, Qingtang Jiang, and Zuowei Shen.
\newblock Image restoration: Wavelet frame shrinkage, nonlinear evolution pdes,
  and beyond.
\newblock {\em Multiscale Modeling \& Simulation}, 15(1):606--660, 2017.

\bibitem{daubechies1992ten}
Ingrid Daubechies.
\newblock {\em Ten lectures on wavelets}, volume~61.
\newblock Siam, 1992.

\bibitem{mallat1999wavelet}
St{\'e}phane Mallat.
\newblock {\em A wavelet tour of signal processing}.
\newblock Elsevier, 1999.

\bibitem{poggio2017and}
Tomaso Poggio, Hrushikesh Mhaskar, Lorenzo Rosasco, Brando Miranda, and Qianli
  Liao.
\newblock Why and when can deep-but not shallow-networks avoid the curse of
  dimensionality: a review.
\newblock {\em International Journal of Automation and Computing},
  14(5):503--519, 2017.

\bibitem{shaham2016provable}
Uri Shaham, Alexander Cloninger, and Ronald~R Coifman.
\newblock Provable approximation properties for deep neural networks.
\newblock {\em Applied and Computational Harmonic Analysis}, 2016.

\bibitem{montanelli2017deep}
Hadrien Montanelli and Qiang Du.
\newblock Deep relu networks lessen the curse of dimensionality.
\newblock {\em arXiv preprint arXiv:1712.08688}, 2017.

\bibitem{wang2018exponential}
Qingcan Wang et~al.
\newblock Exponential convergence of the deep neural network approximation for
  analytic functions.
\newblock {\em arXiv preprint arXiv:1807.00297}, 2018.

\bibitem{he2018relu}
Juncai He, Lin Li, Jinchao Xu, and Chunyue Zheng.
\newblock Relu deep neural networks and linear finite elements.
\newblock {\em arXiv preprint arXiv:1807.03973}, 2018.

\bibitem{martius2016extrapolation}
Georg Martius and Christoph~H Lampert.
\newblock Extrapolation and learning equations.
\newblock {\em arXiv preprint arXiv:1610.02995}, 2016.

\bibitem{sahoo2018learning}
Subham Sahoo, Christoph Lampert, and Georg Martius.
\newblock Learning equations for extrapolation and control.
\newblock In {\em International Conference on Machine Learning}, pages
  4439--4447, 2018.

\bibitem{liu1994weighted}
Xu-Dong Liu, Stanley Osher, and Tony Chan.
\newblock Weighted essentially non-oscillatory schemes.
\newblock {\em Journal of computational physics}, 115(1):200--212, 1994.

\bibitem{shu1998essentially}
Chi-Wang Shu.
\newblock Essentially non-oscillatory and weighted essentially non-oscillatory
  schemes for hyperbolic conservation laws.
\newblock In {\em Advanced numerical approximation of nonlinear hyperbolic
  equations}, pages 325--432. Springer, 1998.

\bibitem{leveque2002finite}
Randall~J LeVeque.
\newblock {\em Finite volume methods for hyperbolic problems}, volume~31.
\newblock Cambridge university press, 2002.

\end{thebibliography}

%
%

%

\end{document}